\documentclass[hidelinks,onefignum,onetabnum]{siamart251216}
\usepackage[applemac]{inputenc}

\usepackage{graphicx,epsfig,color}
\usepackage{amsmath,mathrsfs} 
\usepackage{amssymb} 
\usepackage{amsfonts}
\usepackage{mathtools}

\usepackage{mathtools} 
\usepackage{epstopdf}
\usepackage{ifthen}
\usepackage{bbm}
\graphicspath{{eps/}}
\usepackage{bookmark}
\usepackage{hyperref}
\usepackage{comment}
\usepackage{xfrac}
\excludecomment{extra}
\excludecomment{k}
\includecomment{krad}

\usepackage{algorithm}
\usepackage{algpseudocode}

\newcommand{\inR}{\in \mathbb{R}}

\newcommand{\C}{ \mathbb{C}}
\newcommand{\R}{ \mathbb{R}}

\newcommand{\N}{ \mathbb{N}}

\def\V#1{{\boldsymbol{#1}}}         
\def\Spc#1{{\mathcal{#1}}}  
\def\M#1{{\bf{#1}}}  

\renewcommand{\[}{\begin{equation}}
\renewcommand{\]}[1]{\label{eq:#1}\end{equation}}

\providecommand{\revise}[1]{#1}
\newcommand{\Tr}{\mathsf{T}}

\title{
Universal Architectures for the Learning of Polyhedral Norms and Convex Regularizers\thanks{The research leading to these results has received funding from the European Research Council under Grant ERC-2020-AdG 
FunLearn-101020573.
}
}

\author{
Michael Unser and Stanislas Ducotterd
\thanks{Biomedical Imaging Group, \'Ecole polytechnique f\'ed\'erale de Lausanne (EPFL),
Station 17, CH-1015, Lausanne, Switzerland ({\tt michael.unser@epfl.ch}). }
 }

\begin{document}

\maketitle




%
\begin{abstract}
This paper addresses the task of learning convex regularizers to guide the reconstruction of images from limited data. By imposing that the reconstruction be amplitude-equivariant, we narrow down the class of admissible functionals to those that can be expressed as a power of a seminorm. We then show that such functionals can be approximated to arbitrary precision with the help of polyhedral norms. In particular, we identify two dual parameterizations of such systems: (i) a synthesis (or atomic) form with an $\ell_1$-penalty that involves some learnable dictionary; and (ii) an analysis form with an $\ell_\infty$-penalty that involves a trainable regularization operator.
After having provided geometric insights and proved that the two forms are universal, we propose an implementation that relies on a specific architecture (tight frame with a weighted $\ell_1$ penalty) that is easy to train. We illustrate its use for denoising and the reconstruction of biomedical images.
We find that the proposed framework outperforms the sparsity-based methods of compressed sensing,  while it offers essentially the same convergence and robustness guarantees.
\end{abstract}



%

\section{Introduction}	
The context of this work is computational imaging where the goal is to reconstruct an image
from indirect and possibly incomplete measurements. We focus on scenarios where the measurements are linear functionals
of the object under investigation, such as line integrals in computer tomography \cite{Natterer1984} or samples of its Fourier transform in magnetic resonance imaging \cite{Liang2000principles}.
These quantities are corrupted by noise according to the relation $\M y=\M H \M s + \M n_{\rm noise} \in \R^M$, where
$\M s \inR^d$ is the unknown signal to be recovered and $\M H \in \R^{M \times d}$ is the system matrix that results from the discretization of the physics. Given a candidate signal $\tilde{\M s}$, the forward model is simulated as  $\M H \tilde{\M s}$ which,  for consistency,  ought to match  $\M y$ within the noise level.

Computational imaging 
is currently dominated by two paradigms \cite{McCann2019,Ravishankar2019image}. 
The first one is the variational approach, where the reconstruction task is formulated as a cost minimization problem, as in \eqref{Eq:VarRecontruct}.
Variational techniques have driven the development of image reconstruction algorithms over the past three decades. They are supported by an extensive theory and benefit from the powerful computational tools of convex optimization \cite{Ekeland1999,Boyd2004convex} and compressed sensing \cite{Elad2010b,Chandrasekaran2012,Foucart2013}. Solutions are typically computed iteratively through the minimization of a cost, either by steepest descent or by proximal-gradient techniques \cite{Figueiredo2003,Combettes2005,Figueiredo2007,Beck2009}.
Under standard hypotheses (convexity and coercivity of the cost functional), this framework provides guarantees for convergence \cite{Ekeland1999}, stability \cite{delAguilaPla2023}, and signal recovery from limited measurements \cite{Donoho2006,Candes2007}. 

The second paradigm, at the forefront of research, is the data-driven approach, often referred to as artificial intelligence (AI), where the traditional reconstruction pipeline is replaced or complemented by deep neural networks (DNN)
\cite{Jin2017,Mccann2017Convolutional,Wang2020,Lin2021artificial}. Formally, the resulting reconstruction algorithm is a nonlinear map $\M f_{\V \theta}: \R^M \to \R^d$ with parameters $\V \theta$ (the weights of the neural net). These
are pre-trained for best performance in a regression task such that $\M f_{\V \theta}(\M y_k)\approx \M s_k$ on a representative 
set of data $(\M s_k,\M y_k)_{k=1}^K$, where $K$ (the number of training images) is assumed to be sufficiently large.
DNN-based methods generally achieve superior image reconstructions \cite{Ye2023deep} but suffer from a lack of robustness \cite{Antun2020} and theoretical understanding.
More importantly, they have been found to remove or hallucinate structures \cite{Nataraj2020, Muckley2021MRIChallenge}, which is unacceptable in medical imaging.


Hybrid approaches aim to combine the strengths of both paradigms \cite{Mukherjee2023} and can be broadly
classified in two categories.

{\bf (i) Trainable regularizers}: The guiding principle here is to select a parametric regularization functional 
$\M s \mapsto g_{\V \theta}(\M s)$ and to then optimize its parameters $\V \theta$ (training) for best performance within the variational setting \cite{Hammernik2018,AgHeMaJaco2019}.
An early example is the field of expert (FoE) model \cite{Roth2009}, where $g(\M s)=\sum_{k} \phi_k(\M w_k^\Tr\M s)=\langle \M 1, \V \Phi(\M W \M s)\rangle$ is a sum of ridges. These ridges result from the application of a convex function $\phi_k: \R \to \R^+$ (potential) to some filtered versions of the signal, with the filtering templates $\M w_k$ (collectively encoded in $\M W$) being learned. Enhanced versions of this model allow for trainable potentials under convexity constraints \cite{Chen2017,Nguyen2018,Goujon2024Weakly}. 
A notable alternative  involves using an input-convex neural network \cite{Amos2017} to model the regularizer, as suggested in \cite{Mukherjee2020learned}.

{\bf (ii) Deep unrolling}: The methods in this category typically originate from an existing variational reconstruction algorithm that is reconfigured it into a trainable architecture (unrolling)  \cite{Monga2021}.
This is exemplified in the plug-and-play (PnP) framework \cite{Venkatakrishnan2013plug,Kamilov2023plug}, where deep-learning-based denoisers replace traditional proximal operators \cite{Chan2016plug,Ryu2019plug, Sun2021}. For these methods to remain variational, the learned denoiser must satisfy both firm non-expansiveness (for convergence \cite[Proposition 15]{Hertrich2021}) and monotonicity \cite{Bauschke2012firmly,Gribonval2020}, the latter being challenging to enforce.

This raises a fundamental question: To what extent can one improve variational methods by learning the regularizer, while maintaining their theoretical guarantees? While the FoE model substantially improves upon  ``handcrafted'' regularizers such as total variation (TV) \cite{Goujon2023}, its expressiveness may still be insufficient to capture the full range of convex regularizations. We contend that further improvements are possible within the variational framework under adequate  constraints to avoid hallucinations.

To address these issues, we introduce a general parametric framework for image reconstruction under the constraints of amplitude equivariance (AE) and convexity.
This parameterization involves polyhedral norms and results in a formulation that is reminiscent of compressed-sensing
techniques with a trainable dictionary \cite{Rubinstein2010,Elad2010b}. Our present contribution is two fold: \begin{enumerate}
\item The identification of two dual forms of polyhedral regularization with proof of their universality, in the sense that they are able to encode any AE convex regularizer to an $\epsilon$-level of precision.
\item The transcription of the theory into practical reconstruction architectures, including a constrained version that lends itself well to training.
\end{enumerate}

The paper is organized as follows. In Section \ref{Sec:Formulation}, we formalize the variational reconstruction problem and invoke higher-level principles (amplitude-equivariance and coercivity) to narrow down the class of admissible regularization functionals to some power of a norm, as dictated by Theorem \ref{Prop:seminorms}. In Section \ref{Sec:UnivConvexRegs}, we prove that we can 
achieve universality with the help of polyhedral norms (Theorem \ref{Theo:UniversalApprox}), the latter admitting two dual parameterizations in terms of some linear operator, as stated in Theorem \ref{Theo:ConstructPoly}.
The theoretical foundation for our results is presented in Section \ref{Sec:Theory} with a complete characterization of the geometry of polyhedral Banach spaces (Theorem \ref{Theo:PolyBanach}) in connection with atomic norms \cite{Chandrasekaran2012}. In Section \ref{Sec:CompuPipelines}, we use our results to specify explicit computational architectures, including one that relies on weighted $\ell_1$-minimization and Parseval filterbanks, which is easy to train. We then demonstrate that this approach provides competitive results for denoising and image reconstruction.

%
%
%
%
%
%

\section{Problem Formulation}
\label{Sec:Formulation}

\subsection{Variational Setup for the Resolution of Inverse Problems}
Given the data $\M y \in \R^M$ and a linear measurement model, the generic variational formulation of our signal-reconstruction problem is
\begin{align}
\label{Eq:VarRecontruct}
\min_{\M s \in \R^d} \left(\tfrac{1}{2}\|{\bf y}-{\bf H s}\|_2^2 + \lambda\, g(\M s)\right),
\end{align}
where $\M H \in \R^{M \times d}$ is a known matrix that models the physics of the acquisition, and where $g: \R^d \to \R^+$ is a regularization functional that 
promotes ``regular'' solutions. The relative strength of the regularization is modulated by the regularization factor $\lambda \in \R^+$, according to the standard practice in the field.
In this work, we aim at investigating schemes where $g=g_{\V \theta}$ is a learned regularizer
with trainable parameters collected in $\V \theta$. 

Since our goal is to identify useful classes of regularizers, we can focus on the noise-free scenario (the limit case of \eqref{Eq:VarRecontruct} as $\lambda \to
0$), without loss of generality. The solution of the corresponding generalized interpolation problem reads
\begin{align}
\label{Eq:Geninterpol}
\M f_{\V \theta}(\M y)=\arg \min_{\M s \in \R^d}\left\{ g_{\V \theta}(\M s) \  \mbox{ s.t. } \  \M y= \M H \M s\right\},
\end{align}
where $\M f_{\V \theta}: \R^M \to \R^d$ denotes the underlying (nonlinear) reconstruction operator that implements the minimization and depends on the hyperparameters $\V \theta$ of the regularization functional $g_\V \theta$. The underlying assumption here is that \eqref{Eq:Geninterpol} with $\V \theta$ fixed is well-defined with a unique global minimum. \revise{This is the standard requirement for most works involving convex analysis, including compressed sensing \cite{Donoho2006,Candes2007}.} 


\subsection{Equivariant Regularizers}
Since the measurement operator in \eqref{Eq:VarRecontruct} is linear, a change of amplitude of the signal $\M s$ by the multiplicative factor $\alpha$ results in a corresponding rescaling of the measurements, as in $\alpha \M y= \M H (\alpha\M s)$. Accordingly, we shall insist  that this property
also carry over to the reconstruction specified by \eqref{Eq:Geninterpol} in the noise-free scenario and,
more generally, by \eqref{Eq:VarRecontruct} for $\lambda$ sufficiently small, say, below some critical value $\lambda_0=\lambda(\M y)$. Specifically, when we amplify the measurements and signal by the same factor $\alpha\ne0$, 
we get that
\begin{align*}
\arg \min_{\M s \in \R^d} \tfrac{1}{2}\|\alpha {\bf y}-{\bf H \alpha s}\|_2^2 + \lambda_0\, g_\V \theta(\alpha\M s)=
\arg \min_{\M s \in \R^d} \tfrac{1}{2}\| {\bf y}-{\bf H }\M s\|_2^2 + \tfrac{\lambda_0}{\alpha^2}\, g_\V \theta(\alpha\M s),
\end{align*}
which, according to our requirement, needs to remain compatible with \eqref{Eq:VarRecontruct}.  The desired equivalence is achieved for
$\lambda=\frac{\lambda_0}{\alpha^2}\, \frac{g_\V \theta(\alpha\M s)}{g_\V \theta(\M s)}$, but only under the condition
that the ratio does not 
depend upon $\M s$.
This is to say that the effect of scaling needs to factor out of the regularization, which is a property that we shall refer to as {\em  amplitude equivariance}.

 \begin{definition}
A regularization functional $g: \Spc X \to \R$ is said to be
{\em amplitude-equivariant} if there exists a function $a: \R \to \R$ such that 
 $g(\alpha x)=a(\alpha)g(x)$ for any $x \in \Spc X$ and $\alpha \in  \R \backslash\{0\}$.
\end{definition}

By adding conditions of convexity to ensure that the solution of \eqref{Eq:VarRecontruct} or that of \eqref{Eq:Geninterpol} exist along with symmetry (which is standard in image processing), we now show that we can narrow down the options to $g(\M x)=|p(\M x)|^\gamma$, where $p$ is a seminorm on $\Spc X=\R^d$ and $\gamma\ge 1$.

\revise{To formulate this result rigorously, we start by isolating the structural properties of 
$g$ that are essential for our analysis.}
\begin{definition}
Let $g: \Spc X \to \R$ be a (regularization) functional on a Banach space $\Spc X$. Then, $g$ can be endowed with the following properties.
\begin{enumerate}
\item Convexity: $g\big(\lambda x + (1-\lambda)y\big)\le \lambda g(x) + (1-\lambda)g(y)$ for all $x, y \in \Spc X$ and $\lambda \in[0,1]$.
\item \revise{Coercivity: $g(x)\to +\infty$ as $\|x\|_{\Spc X}\to +\infty$.}
\item Symmetry: $g(x)=g(-x)$ for all $x \in \Spc X$.
\item Unbiasedness: $g(x)\ge g(0)$ for all $x \in \Spc X$.
\item Homogeneity of order $\gamma$ : $g(\alpha x)=|\alpha|^\gamma g(x)$ for all $x \in \Spc X$,
$\alpha \in \R \backslash\{0\}$ and $\gamma \in \R_{>0}$.

\end{enumerate}
\end{definition}
\revise{
Convexity and coercivity are the classical assumptions in convex analysis \cite{Ekeland1999}; together, they guarantee the existence of minimizers for the reconstruction problems \eqref{Eq:VarRecontruct} and \eqref{Eq:Geninterpol}.
The unbiasedness condition ensures that, in the absence of prior information, the preferred reconstruction is the trivial element 
$x=0$.
Finally, homogeneity is the key property that enforces the desired amplitude-equivariant behavior.}
\begin{theorem}
\label{Prop:seminorms}
Let $g:\Spc X \to \R$ be a convex, symmetric (and continuous) functional on the Banach space $\Spc X$. Then, the following statements are equivalent.
\begin{enumerate}
\item The functional $g$ is amplitude-equivariant.
\item The functional $g$ is $\gamma$-homogeneous for some $\gamma\ge 1$.
\item $g(x)=|p(x)|^\gamma$, where $p$ {is a seminorm}
and $\gamma\ge 1$ is the order of homogeneity of $g$.
\end{enumerate}
\end{theorem}
The reason why ``continuity'' is parenthesized in the statement of Theorem \ref{Prop:seminorms} is that the hypothesis is superfluous in the finite-dimensional setting because: (i) all Banach norms, as well as the associated notions of continuity, are equivalent  on $\R^d$ \cite{Megginson2012introduction}; and (ii) all convex functions $g: \R^{d} \to \R$ are 
continuous on their domain \cite{Ekeland1999,Boyd2004convex}. We also note that the covered scenarios all require that $g$ be symmetric and unbiased with $g(0)=0$. 

\begin{proof}
First, we observe that $\gamma$th-order homogeneity (with $\alpha=-1$) implies symmetry. 
The statements in the theorem summarize the following chain of implications, with the first three items covering the classic equivalence: convex $1$-homogeneous functional = seminorm (see Appendix A, Definition \ref{Def:norm1}).
\begin{enumerate}

\item Convexity + 
$1$-homogeneity $\Rightarrow$ subadditivity. 
It suffices to take the convexity inequality with $\lambda=\tfrac{1}{2}$.

\item Subaddivity + 
$1$-homogeneity $\Rightarrow$ convexity: 
$g(\lambda x+(1-\lambda)y)\le g(\lambda  x) +  g\big((1-\lambda) y\big)= \lambda g( x) + (1-\lambda) g(y)$.


\item Continuity and $\gamma$th-order homogeneity $\Rightarrow$ $g(0)=0$:
Given some arbitrary $x_0 \in \Spc X \backslash\{0\}$, we construct the sequence $x_n=\tfrac{1}{n}x_0\ne 0$. It is converging to $0$ in $\Spc X$ because $\|x_n\|_{\Spc X}=\frac{1}{n}\|x_0\|_{\Spc X} \to 0$ as $n \to \infty$.
By invoking the $\gamma$-homogeneity and continuity of $g$, we then get that
$$\lim_{n\to \infty}g(x_n)=\lim_{n\to \infty}\left(\tfrac{1}{n}\right)^{\gamma} g(x_0)=0=g(\lim_{n\to \infty}x_n)=g(0).$$

\item Convexity + symmetry $\Rightarrow$ unbiasedness: Indeed, $g(0)=g\big(\tfrac{1}{2}x+\tfrac{1}{2}(-x)\big)\le \tfrac{1}{2}g(x)+\tfrac{1}{2}g(-x)= g(x)$. 
\item Convexity and $\gamma$th-order homogeneity  $\Rightarrow$ $\gamma\ge 1$: For any $\alpha \in(0,1)$ and $x \in \Spc X$, we must have that
$g(\alpha x + 0)=(\alpha)^\gamma g(x) \le \alpha g(x)$, which implies that $\gamma \ge 1$.

\item Ampliture-equivariance  $\Leftrightarrow$ $\gamma$th-order homogenity: Let $\alpha=\alpha_1 \alpha_2>0$.
Then, $g(\alpha x)=a(\alpha) g(x)=a(\alpha_1)a(\alpha_2) g(x)$, which implies that
$$a(\alpha_1\alpha_1)=a(\alpha_1)a(\alpha_2) \Leftrightarrow \log a(\alpha_1+\alpha_2)=
\log a(\alpha_1)+\log a(\alpha_2).$$
This means that the function
$\log a: \R_{> 0} \to \R$ is linear and hence of the form $\log a(\alpha)=C_0 \alpha$, which is the desired result with $C_0=\log \gamma$. The reverse implication is obvious.
\item Convexity and $\gamma$th-order homogeneity $\Leftrightarrow g(x)=|p(x)|^\gamma$, where $p$ is
convex and $1$-homogeneous.\\Indirect part: The function $g(x)=|p(x)|^\gamma$ is convex because it is the composition of two convex functionals $p: \Spc X \to \R_{\ge0}$ and $|\cdot|^\gamma: \R_{\ge0} \to \R_{\ge0}$. Moreover, it is such that
$g(x)=|p(\alpha x)|^\gamma=|\alpha p( x)|^\gamma=|\alpha|^\gamma g(x)$.\\[1ex]
Direct part: We now show that the convexity of $g(x)\ge 0$ implies that of $p(x)=\big(g(x)\big)^{1/\gamma}$, which is 1-homogeneous.
For any $x,y \in \Spc X$ with $p(x),p(y)\ne 0$ and $\lambda\in [0,1]$ and by letting $P=\lambda p(x) + (1-\lambda) p(y) >0$, we have that 
\begin{align}
\frac{g\big(\lambda x + (1-\lambda) y\big)}{P^\gamma} &= g\left(\frac{\lambda p(x)}{P} \frac{x}{p(x)} + \frac{(1-\lambda) p(y)}{P} \frac{y}{p(y)}\right)\nonumber \\
& \le \frac{\lambda p(x)}{P}  g\left(\frac{x}{p(x)}\right) + \frac{(1-\lambda) p(y)}{P} g\left(\frac{y}{p(y)}\right)\nonumber \\
& \ \ =\frac{\lambda p(x)}{P}  + \frac{(1-\lambda) p(y)}{P} =1,\label{Eq:Convexgamma}
\end{align}
where we also used that $g(x)=p(x)^\gamma$. We then take the $\gamma$th root of \eqref{Eq:Convexgamma}, which yields the desired convexity inequality $p\big(\lambda x  + (1-\lambda) y\big)\le P$. (The latter obviously also holds for 
$p(x)=0$ and/or $p(y)=0$ because of the homogeneity of $p$.)
\end{enumerate}
\end{proof}

In the context of regularization, the amplitude-equivariance property is fundamental 
for it ensures that the recovery procedure is covariant (through a proper adjustment of the regularization strength) to any global rescaling of the input data. 
In view of Theorem \ref{Prop:seminorms}, this reduces the options of acceptable regularizations to (semi)norms.
Since $|\cdot|^\gamma$ is increasing convex, we also note that there is no loss of generality
if we replace $g_{\V \theta}$ in \eqref{Eq:Geninterpol} by $p=|g|^{1/\gamma}$. The same holds true
for \eqref{Eq:VarRecontruct} under a suitable adjustment of the regularization parameter $\lambda$.

If we also want to make sure that the solution of Problem \eqref{Eq:VarRecontruct}
exists, irrespective of the system matrix $\M H$, we need $p$ to be coercive,
which then limits the options to the case where $p$ is a norm.

\section{Universal Equivariant Convex Regularizers}
\label{Sec:UnivConvexRegs}
So far, we have seen that a
 coercive, equivariant convex regularizer $g$ can always be replaced by an appropriate norm $p$ on $\R^d$. Our next step will be to show that this can all be reformulated using convex sets. We then attain universality by approximating the underlying regularization balls as closely as desired with the help of (learned) polytopes. The crucial issue there is to retain the norm property, which is fundamental to our argumentation. 

\subsection{Properties of the Unit Regularization Ball}
A basic result in functional analysis is that every seminorm $p$ on a vector space $\Spc X$ has an equivalent geometric description in terms of the characteristic set:
$B=B_p=\{x \in \Spc X: p(x)\le 1\}$, which happens to be a disk.
The  key to this equivalence is the possibility of recovering $p$ from $B$ with the help of the Minkowski functional, as stated in Theorem \ref{Theo:GaugeNorms}.

\begin{definition}
A subset $B \subset \Spc X$ of a topological vector space $\Spc X$ is called a {\em disk} if it is convex and center-symmetric.
The Minkowski functional (or gauge) associated with such a disk $B$ is 
\begin{align}
\label{Eq:MinkowskiFunct}
\mu_{B}(x)=\inf\{\lambda\in \R_{>0}:  x \in \lambda B\},
\end{align}
with the convention that $\mu_{B}(x)=\infty$ when the infimum in \eqref{Eq:MinkowskiFunct} does not exist.
\end{definition}

A set is said to be absorbing if,
for any $x \in \Spc X$, there exists some $r>0$ such that $x \in \lambda B$ for all $|\lambda|>r$. In particular, if $B$ is a disk, then it is absorbing if and only if it includes the origin as an interior point.

\begin{theorem}[{\cite[p. 120-121,115-154]{Narici2010}
}]
\label{Theo:GaugeNorms}
Let $B$ be an absorbing disk in a topological vector space $\Spc X$.
\begin{enumerate}
\item The functional $\mu_B$ specified by \eqref{Eq:MinkowskiFunct}
is a seminorm on $\Spc X$.
\item If $A=\{x \in \Spc X: \mu_B(x)<1\}$ and $C=\{x \in \Spc X: \mu_B(x)\le1\}$, then
$A\subseteq B \subseteq C$ and $\mu_A=\mu_B=\mu_C$. In particular, if $B$ is open (in the topology of $\Spc X$), then
$A=B$. Likewise, if $B$ is closed, then $B=C$.
\item The gauge $\mu_B$ is continuous on $\Spc X$ if and only if $B$ is a neighborhood of $0$ in $\Spc X$.

\item The gauge $\mu_B$ is a norm on $\Spc X$ if and only if $B$ does not contain any linear subspace of $\Spc X$, except $\{0\}$.
\end{enumerate}
Conversely, let $p$ be a (semi)norm on $\Spc X$. Then, $B=\{x \in \Spc X: p(x) \le1\}$
is convex, balanced, and absorbing with $p=\mu_B$.

\end{theorem}

The powerful aspect of this result is that the Banach disk $B$ does not even need to be closed (see Item 2).
We note that the fourth condition (induction of a norm) is automatically met if
the set $B$ in Theorem \ref{Theo:GaugeNorms}
 is bounded. 
 In such a scenario, $B$ can be assimilated to the unit ball of the norm $\|\cdot\|=\mu_{B}(\cdot)$, while the determination of the norm for a particular point $x \in \Spc X$ (see \eqref{Eq:MinkowskiFunct}) amounts to inflating
 $B$ by $\lambda$ (resp., deflating $B$ if $\lambda<1$) up to the limit point where $x$ enters  $\lambda B$ (resp, $x$ leaves $\lambda B$).
 
\begin{figure}[tpb]
    \centering
  \includegraphics[scale=1]{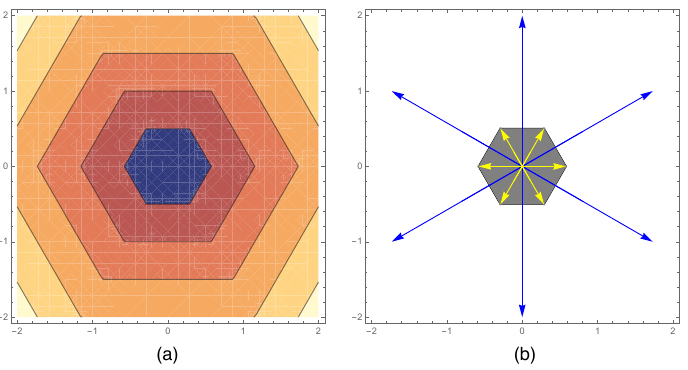}
    \caption{Unit ball of a polyhedral norm: (a) Contour plot of $\M x \mapsto \|\M L \M x\|_{\ell_1}$ with the darker central region in the plane being the unit ball $B_{p_1}$. (b) Overlay of the vertex (yellow) and facet (blue) vectors, the latter being the normals of the supporting hyperplanes of $B_{p_1}$.
 \label{Fig:UnitBall}
}
\end{figure}

To illustrate the concept, we show the contour-plot of the 
regularizer $p_1(\M x)=\|\M L \M x\|_{\ell_1}$ for $\M L=
\left(
\begin{array}{cc}
 0 & 1 \\
 -\frac{\sqrt{3}}{2} & -\frac{1}{2} \\
 \frac{\sqrt{3}}{2} & -\frac{1}{2} \\
\end{array}
\right) \in \R^{3 \times 2}$ in Figure \ref{Fig:UnitBall}. We observe that $B_{p_1}$ is symmetric, convex and bounded, where the latter reflects the property that the underlying regularization operator is injective. These properties are characteristic of a norm, in conformity with Item 4 in Theorem \ref{Theo:GaugeNorms}. The unit ball of this particular norm has
$N=6$ extreme points/vertices $\M g_1,\dots,\M g_N$ (overlaid as light vectors), so that it can also be described as the {\em convex hull} of its vertices: $B_{p_1}={\rm Conv}\{\M g_n\}_{n=1}^N$. We shall see that $B_{p_1}$ also admits a dual representation
as an intersection of half-spaces, whose normals (facet-vectors) are overlaid in blue.
 
 \subsection{Universal Parametric Polyhedral Regularizers}
 Theorem \ref{Prop:seminorms} implies that a convex equivariant regularization $g$ on $\R^d$ can be imposed via the minimization of a corresponding (semi)norm $p=p_\V \theta$. For mathematical convenience, we shall now restrict our attention
to the Banach setting where $p$ is a norm that is characterized by its unit ball $B_p$.
Since the latter is a convex, symmetric, and bounded subset of $\R^d$, it can itself be approximated arbitrary closely by a convex symmetric polytope that has sufficiently many vertices. This process is supported by Theorem \ref{Theo:GaugeNorms}:  to induce an equivalent (semi-)norm, it is sufficient to reproduce the unit ball of the regularizer.
By searching among all convex symmetric polytopes, we are actually spanning
the norms associated with the family of polyhedral Banach spaces (see Section \ref{Sec:PolyBanach}). 
These spaces have some remarkable geometric properties, which are made explicit in Section \ref{Sec:Theory} with the key results being collected in Theorem \ref{Theo:PolyBanach}. 

We now present of a variant of Theorem \ref{Theo:PolyBanach} that offers a practical answer to our design problem: a constructive parameterization of 
all polyhedral norms in terms of dictionaries or, alternatively, regularization operators.

\begin{theorem}[Dual parameterization of all polyhedral norms]
\label{Theo:ConstructPoly}
Let us consider the two complementary functionals
\begin{align}
\label{Eq:SynNorm0}
 \M x \mapsto &\|\M x\|_{1,\M G}=\min_{\M z \in \R^{\tilde N}}\{\|\M z\|_{\ell_1}: \M x=\M G \M z\}
 \\
\M x \mapsto &\|\M x\|_{\M G^\Tr,\infty}=\|\M G^\Tr\M x\|_{\ell_\infty}
\end{align}
parameterized by some rectangular matrix $\M G=[\M g_1 \ \cdots \ \M g_{\tilde{N}}] \in \R^{d \times \tilde{N}}$.
These are norms on $\R^d$ if and only if ${\rm rank}(\M G)=d$, in which case $\Spc X=(\R^d,\|\cdot\|_{1,\M G})$ and $\Spc X'=(\R^d,\|\cdot\|_{\M G^\Tr,\infty})$
form a dual 
pair of polyhedral Banach spaces.

The underlying polyhedral geometry is determined solely by the ``extreme points''
of $\M G$:
\begin{align}
\label{Eq:ExtofMat}
\{\M v_1,\dots,\M v_N\}={\rm Ext}\,{\rm Conv}\{\pm\M g_1,\dots,\pm\M g_{\tilde{N}}\}= {\rm Ext}\, \M G
\end{align}
with $N\le 2\tilde{N}$. 
These are the vertices of the unit ball of $\Spc X$ as well
as the facet vectors of the unit ball of $\Spc X'$, in conformity with the duality/polarity relations $B_{\Spc X'}=B^\circ_{\Spc X}$ (see \eqref{Eq:Polar} in Appendix A).

\end{theorem}

Note that our definition of the extreme points of a dictionary in \eqref{Eq:ExtofMat} includes some $\pm$ signs. This is because the vertices necessarily appear in opposite pairs due to the symmetry of the underlying Banach disks. The process of the determination of the extreme points of a dictionary matrix
$\M G$ is illustrated in Figure 2. The principle at work is that the induced unit regularization ball $B_{\Spc X_{\rm S}}$ is the symmetric convex hull of the dictionary elements.
\begin{figure}[tpb]
    \centering
  \includegraphics[scale=0.5]{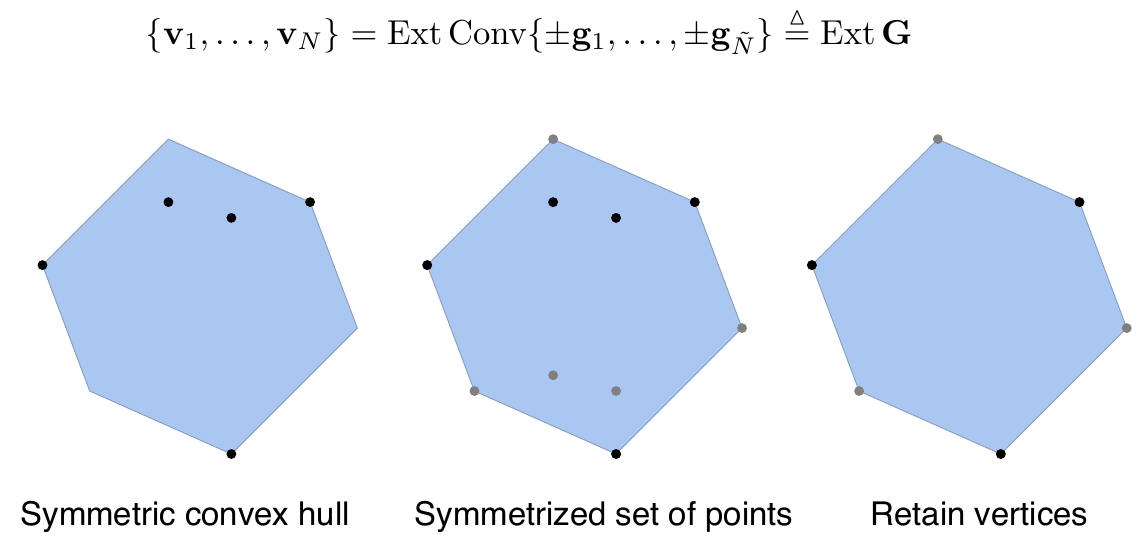}
    \caption{The extreme points of a matrix $\M G$ are given by the vertices of the symmetric convex hull of its column vectors  represented as dark points on the left display.\label{Fig:ExtDico}
}
\end{figure}

Geometrically, the rank condition on $\M G$ in Theorem \ref{Theo:ConstructPoly} translates into the unit ball $B_{\Spc X_{\rm S}}$ (or, equivalently,  
$B_{\Spc X_{\rm A}}$) being a bounded neighborhood of the origin, 
 which is the necessary and sufficient condition to induce a norm (see Theorem \ref{Theo:GaugeNorms}, Item 4). This means that there is actually no restriction and that the scheme is capable of producing any desired 
verticial or facial configuration. 
Equation \eqref{Eq:ExtofMat} also indicates that not all dictionary elements $\M g_n$ are active: the dictionary can be compressed by dropping elements that are not vertices or, otherwise, already represented by their signed counterpart. This reduction process
ends when $N=2\widetilde{N}$, in which case 
the matrix $\M G$ in \eqref{Eq:ExtofMat} is said to be {\em irreducible}.


By concentrating on the case of irreducible matrices and by invoking the property that any bounded set can be approximated to arbitrary precision
by a polytope \cite{Gruber1982,Bronstein2008}, we are able to convert Theorem \ref{Theo:ConstructPoly} into the  universality result of Theorem \ref{Theo:UniversalApprox}.

\begin{theorem}[Universality of polyhedral norm approximations]
\label{Theo:UniversalApprox}
Let $\|\cdot\|_{\Spc X}$ be any norm on $\R^d$. Then, for any $\epsilon \in (0,1)$, there exists
some  $\widetilde{N}\ge d$ and an irreducible matrix $\M F \in \R^{d \times \widetilde{N}}$ 
such that
\begin{align}
\label{Eq:uniFN}
\forall \M x \in \R^d:\quad \frac{1}{1 +\epsilon} \|\M x\|_{\Spc X} \le \|\M F^\Tr \M x\|_{\ell_\infty} \le \frac{1}{1-\epsilon} \|\M x\|_{\Spc X}.
\end{align}
Likewise, there exists some $N\ge d$ and an irreducible matrix $\M V \in \R^{d \times N}$ 
such that
\begin{align}
\label{Eq:uniVN}
\forall \M x \in \R^d:\quad \frac{1}{1 +\epsilon} \|\M x\|_{\Spc X} \le \min_{\M z \in \R^N}\{ \|\M z\|_{\ell_1}: \M x=\M V \M z\} \le \frac{1}{1-\epsilon} \|\M x\|_{\Spc X}.
\end{align}
Moreover, we have that $\epsilon=O\left(\frac{1}{n^{2/(d-1)}}\right)$, where $n= 2N$ 
is the number of extreme points of the underlying polyhedral unit ball $B={\rm Conv}(\M V)$ or $B=\{\M x \in \R^d: \|\M F^\Tr\M x\|_{\ell_\infty}\le 1\}$.
\revise{The same asymptotic rate holds if we substitute $n$ by $\tilde{n}=2\tilde{N}$, which is the number of facets of $B$.}
\end{theorem}
We note that the lower bounds in \eqref{Eq:uniFN} and \eqref{Eq:uniVN} imply that ${\rm rank}(\M F)=d$ and ${\rm rank}(\M V)=d$, respectively.
\begin{proof}
Let $B$ be a convex, center-symmetric subset (i.e., a disk) of a Banach space $\Spc X$ that meets the inclusion conditions
\begin{align}
\label{Eq:BallInc}
B \subseteq C_0 B_{\Spc X}=\{\M x \inR^d: \|\M x\|_{\Spc X}\le C_0\} \quad \mbox{ and } \quad B_{\Spc X} \subseteq \frac{1}{c_0} B,
\end{align}
where $0< c_0 \le  C_0$ are two constants. The right-hand side of \eqref{Eq:BallInc} implies that
$B$ is absorbing, while the left-hand side tells us that $B$ is bounded. This ensures that
 $\mu_B$ is a valid norm on $\Spc X$ (by Theorem \ref{Theo:GaugeNorms}). We then exploit the homogeneity of this norm to rewrite \eqref{Eq:BallInc} as
\begin{align}
\label{Eq:EqNorms}
\forall x \in \Spc X,\quad c_0 \|x\|_{\Spc X}  \le \mu_B(x) \le C_0 \|x\|_{\Spc X},\end{align}
which is the canonical way to indicate that the two norms are topologically equivalent.

In our case of interest, $\Spc X=\R^d$ and $B$ is the polyhedral unit ball induced by $\M F$ or $\M V$, with $\M V$ having the capacity to encode all symmetric polytopes with $n=2N$ vertices. 
Fixing $n$ and viewing $B_{\Spc X}$ as a bounded convex set, we can then invoke a classical result \cite{Bronshteyn1975,Bronstein2008} that states the existence of an approximating polytope $B$ such that 
\begin{align}
\label{Eq:ErrorBoundPol}
d_H(B_{\Spc X},B)= \max\left\{{\sup_{\M x \in B_{\Spc X}} \inf_{\M y \in B}} \|\M x -\M y\|_2, \sup_{\M y \in B} \inf_{\M x \in B_{\Spc X}} \|\M x -\M y\|_2\right\} 
&\le  \frac{C_{\Spc X}}{n^{2/(d-1)}},
\end{align}
where $d_H(B_{\Spc X},B)$ is the Hausdorff distance between  $B_{\Spc X}$ and $B$, 
and $C_{\Spc X}$ is a constant that depends solely on $B_{\Spc X}$. This means that
one can make $d_H$ arbitrarily small by selecting a dictionary that has sufficiently many extreme points or, equivalently, a regularization operator $\M F^\Tr$ that induces sufficiently many facets (see Theorem \ref{Theo:PolyBanach} for additional insights).

Since all finite-dimensional norms are equivalent, the estimate in \eqref{Eq:ErrorBoundPol} also holds if we substitute $\|\cdot\|_{2}$ 
by any other norm under some adaptation of the bounding constant $C_{\Spc X}$. The form that is suitable for our purpose is
\begin{align*}
\tilde d_H(B_{\Spc X},B)&=\max\left\{\sup_{\M x \in B_{\Spc X}} \inf_{\M y \in B} \|\M x -\M y\|_{\Spc X},\, \sup_{\M y \in B} \inf_{\M x \in 
B_{\Spc X}} \|\M x -\M y\|_{\Spc X }\right\},
\\
&= \min \{\lambda\ge 0: B_{\Spc X} \subset B + \lambda B_{\Spc X}, B \subset B_{\Spc X} + \lambda B_{\Spc X}\}.
\end{align*}
where the ``+" symbol applied to sets denotes the Minkowski sum.
The second form implies that $B \subset (1 + \tilde d_H)B_{\Spc X}$ and $(1-\tilde d_H) B_{\Spc X} \subset B$ under the implicit assumption that $0\le \tilde d_H<1$. We make the link with \eqref{Eq:BallInc}
by setting $c_0=\frac{1}{1+\tilde d_H}$ and $C_0=\frac{1}{1-\tilde d_H}$, with \eqref{Eq:EqNorms} then  yielding the desired result.
\end{proof}

\revise{Since $\epsilon$ in Theorem \ref{Theo:UniversalApprox} can be arbitrarily small, we can push the scheme by growing the dictionary 
to achieve an isometry; that is, an exact norm equivalence.
As for the last statement in Theorem \ref{Theo:UniversalApprox}, it is a standard result in the theory of polytopes which reflects the property that approximation quality is governed by surface area partitioning \cite{Gruber1982,GRUBER1993}.
This favors the choice of ``regular'' approximating polytopes for which the number of vertices $n$ and facets $\tilde{n}$ are proportional \cite{Gruber1998}. 
It is also known that the exponent $2/(d-1)$
is optimal and cannot be improved in general, while the proportionality constant depends on the affine surface area of the convex body $B_\Spc X$ being approximated \cite{GRUBER1993}.
 }

\section{Supporting Banach-Space Theory}
\label{Sec:Theory}
\subsection{Nomenclature and notation}

We start by recalling a few geometric and topological concepts.
Let $C \subset \R^d$ denote some proper subset of $\R^d$. Then, $C$ said to be {\em convex} if and only if $\lambda \M x + (1-\lambda) \M y \in
C$ for all $\M x, \M y \in C$ and $\lambda \in [0,1]$. Likewise, it is said be 
{\em bounded} if there exists an Euclidean neighborhood $N(\M x_0;R)=\{\M x \inR^d: \|\M x-\M x_0\|_2\le R\}$ of radius $R>0$ and center $\M x_0$ such that 
$C \subset N(\M x_0;R)$. The set $C$ is (topologically) {\em closed} if it contains all its boundary points or, equivalently, if its complement in $\R^d$ is open.
Since all finite-dimensional norms are equivalent, the closedness of $C$ can be formalized as follows: for any $\M y \notin C$, there exists some Euclidean neighborhood $N(\M y,\epsilon)$ with $\epsilon>0$ whose intersection with $C$ is empty. 
In finite dimensions, there is also an equivalence between the mathematical notion of {\em compactness} and the property of being closed and bounded  (Heine-Borel theorem).

A {\em disk} $D$ is a convex, center-symmetric (or balanced) subset of $\R^d$. The prototypical example of a {\em compact disk}
is the unit ball $B_{\Spc X}=\{\M x\in \R^d: \|\M x\|_{\Spc X}\le 1\}$ associated with any admissible norm $\|\cdot\|_{\Spc X}$ on $\R^d$.

A {\em polyhedron} $Q \subset \R^d$ is the intersection of a finite number ($M$) of half-spaces. Formally,
$Q=\bigcap_{m=1}^M K(\M f_m,\alpha_m)$,
where \begin{align}
K(\M f_m,\alpha_m)=f\{\M x \in \R^d: \langle \M f_m, \M x \rangle\le \alpha_m\}
\end{align}
is a half-space characterized by its outward normal vector $\M f_m \in \R^d$ (which may or may not be normalized) and an offset $\alpha_m \in \R$. The representation is called
{\em irreducible} when the number of half-spaces $M$ is minimal. In such a scenario, the corresponding hyperplanes 
\begin{align}
H(\M f_m,\alpha_m)=\{\M x \in \R^d: \langle \M f_m, \M x \rangle= \alpha_m\}
\end{align}
outline the boundary of $Q$ and are called the {\em supporting hyperplanes} of the polyhedron. The boundary of the polyhedron can then be subdivided into $M$ facets $F_m=Q\cap H(\M f_m,\alpha_m)$, each of which is the intersection of the body and its corresponding supporting hyperplane. 

A {\em convex polytope} $P$ is the {\em convex hull} of a finite number ($K$) of points, which is written as
$P={\rm Conv}\{\M p_{1},\dots,\M p_{K}\}$ with $\M p_{k} \in \R^d$. Alternatively, $P$ can be described as a convex compact subset of $\R^d$
that has a finite number of {\em extreme points}. These extreme points are ${\rm Ext}P=\{\M v_1,\dots,\M v_{N}\}$ with $N\le K$; they are the {\em vertices} of the polytope. They necessarily form a subset of the  $\M p_{k}$'s in the first definition and yield the most concise (irreducible) representation of the polytope as $P={\rm Conv}\{\M v_n\}_{n=1}^{N}$.

The fundamental theorem for polytopes states that that every bounded polyhedron is a convex polytope and vice versa \cite{Ziegler2012Polytopes}.
In fact, as we shall see, the two representations are related by duality, which also means that it is always possible to switch from one to the other and to use the one that is the most appropriate for a given task or mathematical derivation.

To make the link with the geometrical formulation of Chandrasekaran et al. \cite{Chandrasekaran2012}, we also recall their definition of the
{\em atomic norm} induced by a dictionary $\Spc A=\{\M a\}$ of atoms $\M a \in \R^d$:
\begin{align}
\label{Eq:Atomic}
\|\M x\|_{\Spc A}=\inf\left\{\sum_{\M a \in \Spc A} c_{\M a}: \M x =\sum_{\M a \in A} c_{\M a}\M a, \ c_{\M a}\in \R_{\ge 0}\  \forall \M a \in \Spc A\right\}.
\end{align}

\subsection{Polyhedral Banach Spaces}
\label{Sec:PolyBanach}
We now present the theoretical framework that supports the derivation of Theorem \ref{Theo:ConstructPoly} 
through the characterization of all Banach spaces whose unit balls have a finite number of extreme points.

Klee calls a Banach space $\Spc X$ polyhedral if the unit ball of every finite-dimensional subspace of $\Spc X$ is a polyhedron \cite{Klee1960}. After having identified $c_0(\N)$ as the (infinite-dimensional) prototype of such spaces, he asked whether there also exist some infinite-dimensional polyhedral spaces that are reflexive. Lindenstrauss then showed the impossibility of such a construction by proving that there is no infinite-dimensional conjugate space $\Spc X'$ that is polyhedral \cite{Lindenstrauss1966}. In particular, this implies that $\ell_1(\N)=(c_0(\N))'$ 
is not polyhedral in the strict sense\footnote{While $\ell_1(\N)$ has an infinite number of extreme points, these are countable, which sets this space apart from most other Banach spaces---in particular, the strictly-convex ones.} of the term.

Fortunately for us, the situation is much more favorable in finite dimensions where all Banach spaces are reflexive and where the duals of polyhedral spaces are polyhedral as well.

\begin{definition}
\label{Def:PolyhedralSpace}
A finite-dimensional Banach space $\Spc X$ is said to be {\em polyhedral} if its unit ball $B_{\Spc X}$ has a finite number of extreme points.
\end{definition}
Since $B_{\Spc X}$ is compact 
by construction, it is the convex hull of its extreme points and, hence, a polyhedron (or polytope), which justifies the nomenclature. 
In the geometric context of polytopes, these extreme points are called vertices, and one can then take advantage of all the mathematical tools available in this area.

The next theorem expresses the remarkable consequences of the association between the vertices (resp., the facets) of $B_{\Spc X}$ and the facets (resp., vertices) of the unit ball of its topological dual. Most of the results in Theorem \ref{Theo:PolyBanach} are standard in the theory of convex polytopes \cite{Ziegler2012Polytopes, Grunbaum2000Polytope,Brondsted2012Polytopes}, 
even though they are rarely stated in the present format (Banach setting). 
This is not so for the S-form (Item 4), which coincides with the atomic norm of the dictionary
formed by the vertices of $B_{\Spc X}$ because its vertices come in signed pairs.
The relation $\mu_{{\rm Conv}(\Spc A)}=\|\cdot\|_{\Spc A}$ with $\Spc A={\rm Ext}(B_{\Spc X})$ is given  in \cite{Chandrasekaran2012} and identified as a special case of Bonsall's general decomposition theorem for Banach spaces \cite{Bonsall1991}.
Given the crucial role of these equivalences in our analysis, we are providing a self-contained proof to clarify and reinforce the concepts.

\begin{definition}
A collection $\{\M v_1,\dots, \M v_N\}$ of vectors (vertices or facet-normals) is said to be {\em irreducible}
if ${\rm Ext}\,{\rm Conv}\{\M v_1,\dots, \M v_N\}=\{\M v_1,\dots, \M v_N\}$.
\end{definition}
If a given collection of vectors (or dictionary elements) does not meet the above condition, then one can apply a standard procedure to reduce it to a (unique) minimal set, which is formed by the vertices of its convex hull.



\begin{theorem}[Geometry of polyhedral Banach spaces]
\label{Theo:PolyBanach}
Let $\Spc X=(\R^d,\|\cdot\|_{\Spc X})$ and $\Spc X'=(\R^d,\|\cdot\|_{\Spc X'})$ be a dual pair of polyhedral Banach spaces.
Then, the Banach topology of $\Spc X$ has the following equivalent geometric descriptions in terms of (irreducible) vertices or facet-vectors.
\begin{enumerate}
\item $\|\cdot\|_{\Spc X}=\mu_{B_{\Spc X}}(\cdot)$ where the polyhedral unit ball $B_{\Spc X}$ is the convex hull of its vertices $\{\M v_1,\dots,\M v_N\}={\rm Ext} B_{\Spc X}$ with $N> d$, which leads to\begin{align}
\label{Eq:VertexUnitBall}
B_{\Spc X}&={\rm Conv}\{\M v_1,\dots,\M v_N\} \nonumber \\
&=\{\M x=\sum_{n=1}^N \lambda_n \M v_n: \lambda_n\ge 0 \mbox{ for all } n  \mbox{ and } \sum_{n=1}^N \lambda_n=1 \}.
\end{align}
\item $\|\cdot\|_{\Spc X}=\mu_{B_{\Spc X}}(\cdot)$ where $B_{\Spc X}$ is an intersection of $M$ half-spaces given by
\begin{align}
\label{Eq:FacialUnitBall}
B_{\Spc X}&=\bigcap_{m=1}^M K(\M f_m,1)=\{\M x \in \R^d: \langle \M f_m, \M x \rangle\le 1\mbox{ for all } m\},
\end{align}
with the $\M f_m$ (facet-vectors) being the directed normals of the facets $F(\M f_m)=H(\M f_m,1)\cap B_{\Spc X}$ of $B_{\Spc X}$.
\item Analysis form of the norm with $\M F=[\M f_1 \cdots \M f_M]\in \R^{d \times M}$:
\begin{align}
\|\M x\|_{\Spc X}=\max\{|\langle \M f_m,\M x\rangle|\}_{m=1}^M=\|\M F^\Tr \M x\|_{\ell_\infty}.
\label{Eq:AnalNorm}
\end{align}
\item Synthesis form of the norm with $\M V=[\M v_1 \cdots \M v_N]\in \R^{d \times N}$:
\begin{align}
\|\M x\|_{\Spc X}=\min_{\M z \in \R^N}\{\|\M z\|_{\ell_1}: \M x=\M V\M z\}
\label{Eq:SynNorm}.
\end{align}
\end{enumerate}
The same relations hold for the dual norm with the role of the vertices and facets being interchanged. For instance, we have that
\begin{align}
B_{\Spc X'}&={\rm Conv}\{\M f_1,\dots,\M f_M\} \nonumber\\
&=\bigcap_{n=1}^N K(\M v_n,1)=\{\M y \inR^d: \|\M y\|_{\Spc X'}=\|\M V^\Tr  \M y\|_{\ell_\infty}\le 1\}.
\end{align}
This duality is examplified by the relation
\begin{align}
\label{Eq:FacetIneq}
\forall (\M x,\M f_m) \in B_{\Spc X} \times {\rm Ext}B_{\Spc X'}, \quad -1\le \langle \M f_m, \M x \rangle\le 1
\end{align}
with the upper/lower bound being achieved if and only if $\M x \in F(\pm\M f_n)$ with $F(\M f_n)=B_{\Spc X}\cap H(\M f_m,1)$ being the facet with normal vector $\M f_m$.
%
%
%
%
%
\end{theorem}

\begin{proof}
\item Item 1. We can either view \eqref{Eq:VertexUnitBall} as the definition of a convex polygon or as a consequence of the Krein-Milman theorem which states that any compact convex set is the convex hull of its extreme points, with the latter being the vertices of $B_{\Spc X}$. One then recovers the norm by taking the Minkowski functional of $B_{\Spc X}$ (see Theorem \ref{Theo:GaugeNorms}).\vspace{1ex}
\item 1 $\Leftrightarrow$ 2. Here, we invoke the fundamental theorem for polytopes \cite[Theorem 1.1, p. 29]{Ziegler2012Polytopes} which states the equivalence between the vertex-based
and facial representations of a polytope.\vspace{1ex}

\item 2 $\Leftrightarrow$ 3. The symmetry of the $B_{\Spc X}$ (Banach disk) implies that its facet-vectors necessarily come in (signed) pairs $\pm\M f_n$. We then identify the unit ball $\{\M x \in \R^d: \|\M F^\Tr \M x\|_{\ell_\infty}\le1 \}$ as the convex set specified by 
\eqref{Eq:FacialUnitBall} by observing that $K(\M f_m,1)\cap K(-\M f_m,1)=\{\M x \in \R^d: |\langle \M f_m, \M x \rangle|\le 1\}$. Since the unit balls are the same, the induced norms are identical, as direct consequence of
Theorem \ref{Theo:GaugeNorms}.\vspace{1ex}

\item 1 $\Leftrightarrow$ 4. For any $\M x \in \R^d\backslash \{\M 0\}$, we have that $\frac{\M x}{\|\M x\|_{\Spc X}} \in B_{\Spc X}$ which,
in view of \eqref{Eq:VertexUnitBall}, implies the existence of some $\M z=\|\M x\|_{\Spc X}\,(\lambda_1,\dots,\lambda_N) \in \R^N$
such that $\M x=\M V \M z$ with $\|\M z\|_{\ell_1}=\|\M x\|_{\Spc X} \sum_{n=1}^N |\lambda_n|= \|\M x\|_{\Spc X}$. This proves that the functional $p(\M x)=\min\{\|\M z\|_{\ell_1}: \M x=\M V \M z\}$ is well-defined with $p(\M x)\le \|\M x\|_{\Spc X}$.
Conversely, let $\M x=\sum_{n=1}^N z_n \M v_n$. Then, by the triangle inequality, we have that
$\|\M x\|_{\Spc X}\le \sum_{n=1}^N |z_n| \|\M v_n\|_{\Spc X}=\sum_{n=1}^N |z_n|\times 1=\|\M z\|_{\ell_1}$. This also hold for the infimum configuration so that $\| \M x\|_{\Spc X}\le p(\M x)$. By combining these two inequalities, we get $p(\M x)=\| \M x\|_{\Spc X}$.\vspace{1ex}

\item Duality between vertices and facets: This is a central theme in the theory of polytope that goes under the name of polarity \cite[Theorem 9.1, p. 57]{Brondsted2012Polytopes}.  When applied to the present setting, the polarity theorem for polytopes implies that $B_{\Spc X'}=B_{\Spc X}^\circ=\bigcap_{n=1}^N K(\M v_n,1)$ is the polar of $B_{\Spc X}={\rm Conv}\{\M v_1,\dots,\M v_N\}$, while $B_{\Spc X}=B_{\Spc X}^{\circ\circ}=B_{\Spc X'}^\circ=\bigcap_{m=1}^M K(\M f_m,1)$ is the polar of 
$B_{\Spc X'}={\rm Conv}\{\M f_1,\dots,\M f_M\}$. (In the finite-dimensional setting, all Banach spaces are reflexive with $B_{\Spc X}=B_{\Spc X''}$, which makes the definition of polarity---see \eqref{Eq:Polar} in the appendix---compatible  with that of Banach duality.) In view of the argumentation in Item 3, the polyhedral representation
of $B_\Spc X$ is equivalent to
$B_{\Spc X}=
\{\M x \in \R^d: \max\{|\langle \M f_m, \M x \rangle|\le 1\}_{m=1}^M$, which implies \eqref{Eq:FacetIneq}.
The equality holds for $\langle \M f_m, \M x \rangle= \pm 1$, which happens if and only if
$\M x \in B_{\Spc X}\cap H(\M f_m,1)$ or 
$\M x \in B_{\Spc X}\cap H(-\M f_m,1)=F(-\M f_m)$, where $H(\M f_m,1)$ is the supporting hyperplane with outward normal vector $\M f_m$.
%
\end{proof}

The geometric interpretation of \eqref{Eq:FacetIneq} is that $\M f_m$, 
which is now identified as a vertex of $B_{\Spc X'}$, 
is the common Banach conjugate\footnote{The general duality inequality for Banach spaces is $|\langle x,y\rangle|\le \|x\|_{\Spc X}\|y\|_{\Spc X'}$ for all $(x,y) \in \Spc X \times  \Spc X'$. Two such elements are said to be Banach conjugates if they saturate the inequality with $\langle x,y\rangle=\|x\|_{\Spc X}\|y\|_{\Spc X'}$.} of all boundary points lying on the corresponding facet of $B_{\Spc X}$.

An important observation is that \eqref{Eq:AnalNorm} and \eqref{Eq:SynNorm} also hold for ``non-minimal'' sets of vertices/facets.
Indeed, the enabling property that yields the equivalence of norms in the proof of Theorem \ref{Theo:PolyBanach} is the preservation of the underlying convex hull(s). This naturally leads to the statement of the more practical variant of the characterization in Theorem 
\ref{Theo:ConstructPoly}. 
Likewise, because of the presence of the absolute value in the underlying norms, it is possible to 
reduce the width of the matrices $\M V$ and $\M F$ in \eqref{Eq:AnalNorm} and  \eqref{Eq:SynNorm} by two, so that
\begin{align}\|\M x\|_{\Spc X}&=\min_{\M z \in \R^M}\{\|\M z\|_{\ell_1}: \M x=\M V_+\M z\}\nonumber\\&
=\|\M F_+^\Tr \M x\|_{\ell_\infty}
\end{align}
with $\M V_+\inR^{d \times N/2}$ and $\M F_+\inR^{d \times M/2}$, which are the two concrete forms that will be referred to as ``irreducible.''

It is instructive to apply Theorem \ref{Theo:PolyBanach} to $\Spc X=\ell_1(\mathbb{I})=(\R^N,\|\cdot\|_{\ell_1})$. The corresponding unit ball is the famous cross-polytope, which
has the following two equivalent representations: \begin{align}
B_{\ell_1}&={\rm Conv}\{\pm\M e_1, \dots, \pm \M e_N\} \label{Eq:Crossp1}\\
&=\big\{\M x\in \R^N: \sum_{n=1}^N b_n x_n\le 1 \mbox{ for all } b_1,\dots,b_N \in \{-1,+1\}\big\} \label{Eq:Crossp2},
\end{align}
where $\M e_n$ denotes the $n$th element of the canonical basis.
Equation \eqref{Eq:Crossp1} explicitly lists its $2N$ vertices (extreme points),
while \eqref{Eq:Crossp2} shows the $M=2^N$ supporting planes that specify its facets.
In this basic scenario, the S-form of the norm is simply $\|\M x\|_{\Spc X}=\|\M x\|_{\ell_1}$, while its A-form is \begin{align}
\|\M x\|_{\ell_1}=\|\M B^\Tr_{N}\M x\|_{\ell_\infty},
\label{Eq:l1infty}
\end{align}
where $\M B_{N}=[\M b_1\ \cdots \  \M b_{2^N}] \in \R^{N \times 2^N}$ is the matrix representation of all ``binary''
vectors $\M b_m$ in \eqref{Eq:Crossp2} with components in $\{-1,1\}$.
The theoretical relevance of \eqref{Eq:l1infty} is the possibility of representing an $\ell_1$-norm by an $\ell_\infty$-norm, while the reverse generally does not hold for $N>2$ (see Appendix B).

The dual of $\ell_1(\mathbb{I})$ is $\Spc X'=\ell_\infty(\mathbb{I})=(\R^N,\|\cdot\|_{\ell_\infty})$
whose unit ball is the $N$-dimensional hypercubic $B_{\ell_\infty}$, 
which has $2^N$ binary vertices $\M b_m$ and $2N$ facets with normal vectors $\pm\M e_n$. In this scenario, it is the S-form 
\begin{align}
\label{Eq:inftyl1}
\|\M x\|_{\ell_\infty}=\min_{\M z \in \R^{2^N}}\{\|\M z\|_{\ell_1}: \M x= \M B_N \M z\},
\end{align}
that is non-standard,
while the A-form reverts to the classical formula.

\section{Parametric Reconstruction Pipelines}
\label{Sec:CompuPipelines}
Given the constraint of an equivariant regularizer $g$, we now use
Theorems \ref{Theo:ConstructPoly} and \ref{Theo:UniversalApprox} to specify some 
polyhedral approximation(s) of the generic variational signal-reconstruction problem \eqref{Eq:VarRecontruct}.
We shall consider three architectures, with the third being the one that presently lends itself the best to training.
\subsection{Sparse Encoding in a Frame/Dictionary}
The first form in Theorem \ref{Theo:ConstructPoly}, together with the universality result \eqref{Eq:uniVN}, suggests the reformulation of our problem as
\begin{align}
\label{Eq:InvDico1}
S_{\rm I}=\{\tilde{\M s} \in \arg \min_{\M s \in \R^d} \left(\tfrac{1}{2}\|\M y - \M H \M s\|_2^2 + \lambda \|\M s\|_{{\rm S}, \M G}\}\right),
\end{align}
where the underling regularizer---the synthesis/atomic norm defined by \eqref{Eq:SynNorm0}---is para\-meterized
by the matrix $\M G \inR^{d \times N}$. The literal form \eqref{Eq:InvDico1} involves two nested optimizations: an outer one to minimize the overall cost, and an inner one to evaluate $\|\M s\|_{{\rm S}, \M G}$ for every candidate $\M s$, which, according to the definition, requires
the determination of an optimal codeword 
in the dictionary $\M G$. Fortunately, we can simplify the
task by setting $\M s=\M G\M z$ and recasting \eqref{Eq:InvDico1} in terms of the auxiliary codeword variable $\M z \in \R^N$. After substitution and mutualization of the minimization, this results in
\begin{align}
\label{Eq:InvDico2}
S_{\rm I}=\{\tilde{\M s}=\M G \tilde{\M z}: \  \tilde{\M z} \in \arg \min_{\M z \in \R^N} \tfrac{1}{2}\|\M y - \M H \M G \M z\|_2^2 + \lambda \|\M z\|_{\ell_1}\},
\end{align}
which has a strong feel of ``déja vu.''
Indeed, \eqref{Eq:InvDico2} is the standard ``synthesis form'' (a.k.a.\ sparse encoding) used in compressed sensing \cite{Elad2010b}, for which a multitude of efficient iterative solvers have been developed. The present contribution is the proof that \eqref{Eq:InvDico2} with a learnable dictionary $\M G$, as advocated in \cite{Rubinstein2010},
is universal for the class of equivariant convex regularizers.

We have already hinted to the property that the matrix $\M G=[\M g_1 \cdots \M g_N]$ 
defines a dictionary with atoms $\M g_n \in \R^d$. The admissibility condition ${\rm rank}(\M G)=d$ in Theorem \ref{Theo:ConstructPoly}
implies that the $\M g_n$ form a frame of $\R^d$  \cite{Christensen2003}. This means that
there exist two constants $A$ and $B$ such that
\begin{align}
\label{Eq:Frame}
\M s \in \R^d:\quad A \|\M s\|^2_2 \le \sum_{n=1}^N |\langle \M s, \M g_n\rangle|^2 \le B \|\M s\|_2^2.
\end{align}
The optimal frame bounds $A$ and $B$ are the minimum and maximum eigenvalues of the frame operator $\M s \mapsto \sum_{n=1}^N \M g_n\langle \M s, \M g_n \rangle =\M A \M s$, with $\M A=\M G\M G^\Tr \in \R^{d \times d}$ \cite{Christensen2003}. 

To ensure that all the elements of the dictionary are extreme points (irreducible scenario), we propose to impose the normalization constraint $\|\M g_n\|_2=1$.
Indeed, the strong convexity of $\M z \mapsto \|\M z\|_2^2$ implies that
$\|\lambda \M g_{n} + (1-\lambda) \M g_{n'}\|_2^2 <  \lambda \|\M g_n\|^2_2 + (1-\lambda) \|\M g_{n'}\|^2_2=1$,
for any $\lambda \in\, ]0,1[$ and $\M g_{n}\ne \M g_{n'}$. This means that all non-extreme points $\M g \in {\rm conv}\{\M \pm \M g_1,\dots,\pm \M g_N\}$ that are in the symmetric convex hull of the $\M g_n$ have an $\ell_2$-norm that is smaller than $1$. We can then insert a new element $\M g_{N+1}$ with $\|\M g_{N+1}\|_2=1$ without spoiling the extreme-point properties of the others.

While \eqref{Eq:InvDico2} is very attractive for inference because of the availability of efficient resolution methods when the dictionary is known, it is much harder to optimize over $\M G$ in order to learn the ``optimal" dictionary for a given class of signals, also in the simpler denoising scenario  ($\M H=\M I_d$) commonly used for training. \revise{We attribute this state of affairs to the fact that $\M G$ has a huge null space (of dimension $N-d$), which makes
the optimization difficult and costly. There is also the necessity to maintain the condition ${\rm rank}(\M G)=d$ to ensure that \eqref{Eq:InvDico2} with $\M G$ fixed is well-posed. This observation aligns with prior reports that convolutional sparse coding (CSC)  \cite{Papyan2017,ReyOtero2020}---the special case of \eqref{Eq:InvDico2} with $\M =\M I$ and $\M G$ shift-invariant---is not among the most effective dictionary-based denoising techniques \cite{Chen2014Learning,Carrera2017,Vedaldi2020}. }

\subsection{Universal Weighted $\ell_\infty$-Regularization}
The optimization problem associated with the second form in Theorem \ref{Theo:ConstructPoly} is
\begin{align}
\label{Eq:InvReg1}
S_{\rm II}=\{\tilde{\M s} \in \arg \min_{\M s \in \R^d} \tfrac{1}{2}\|\M y - \M H \M s\|_2^2 + \lambda \|\M F^\Tr\M s\|_{\ell_\infty}\}
\end{align}
where we are now using a distinct symbol for the regularization matrix $\M F=[\M f_1 \cdots \M f_{\widetilde{N}}] \inR^{d \times \tilde{N}}$  formed from the facet vectors $\M f_n$ of the unit ball $B_{\Spc X_{\rm S}}$. Here too, \eqref{Eq:InvReg1} is reminiscent of the regularization methods used in compressed sensing, with the important difference that the regularizer now involves an $\ell_\infty$-norm. Since the proximity operator $\|\cdot\|_{\ell_\infty}$ (adaptive soft-clip) is (almost) as convenient as that of
$\|\cdot\|_{\ell_1}$ (soft-threshold), it is possible to adapt all the traditional algorithms (e.g., FISTA or ADMM) for the resolution of \eqref{Eq:InvReg1} with $\M F^\Tr$ fixed. 

Observe that \eqref{Eq:InvReg1} is conceptually simpler and less constraining than \eqref{Eq:InvDico2} because there is no (higher-dimensional) auxiliary variable $\M z$ nor any issue with some hidden null space. 
In fact, the scheme has the capacity to encode all seminorms if one drops the rank constraint on $\M F$.


As for training, it can in principle be achieved based on a classic regression task where $\M F^\Tr$ is optimized for the minimum mean-square error denoising ($\M H=\M I_d$)
of a representative set of images.
The caveat is that the back-propagation of the gradient (training error)
to the regularization weights will be very slow because of the underlying $\ell_\infty$-norm acting as a binary switch. Finding a way around this limitation is a topic that deserves further investigation because
\eqref{Eq:InvReg1} is our most expressive convex architecture. 

\subsection{Weighted $\ell_1$-Regularization Revisited}
Due to the difficulties encountered with the training of the two aforementioned architectures, \revise{we  decided to revisit the regularization setup more typical of compressed sensing \cite{Candes2011compressed,Lin2014}}; namely,
 \begin{align}
\label{Eq:L1reg}
\tilde{S}_{\rm I}= \arg \min_{\M s \in \R^d} \tfrac{1}{2}\|{\bf y}-{\bf H s}\|_2^2 + \lambda \|\M L \M s\|_{\ell_1},
\end{align}
which involves an operator $\M L
 \in \R^{N \times d}$ with $N\ge d$ and a  $\ell_1$-penalty instead of the less standard $\ell_\infty$-norm in \eqref{Eq:InvReg1}. 
We now show that \eqref{Eq:L1reg} with trainable $\M L$ constitutes a valid alternative to \eqref{Eq:InvDico1} or \eqref{Eq:InvReg1} in the sense that it
 also induces a polyhedral regularization. \revise{For completeness, we also identify the underlying dual Banach geometry, which is analogous to the first part of Theorem \ref{Theo:ConstructPoly}, with the role of the $\ell_1$ and $\ell_\infty$-norm being interchanged.}

\revise{\begin{theorem}
\label{Theo:L1norm}
Let us consider the two complementary functionals
\begin{align}
\label{Eq:AnalNorm0}
 \M s \mapsto &\|\M s\|_{\M L,1}=\| \M L \M s\|_{\ell_1}
 \\
\M s \mapsto &\|\M s\|_{\infty, \M L^\Tr}=\inf\bigl\{\|\M t\|_{\ell_\infty} : \M s=\M L^\Tr \M t=\sum_{n=1}^N \M u_nt_n\bigr\}
\end{align}
parameterized by some rectangular matrix $\M L=[\M u_1 \cdots \M u_N]^\Tr  \in \R^{N \times d}$.
These are norms on $\R^d$ if and only if ${\rm rank}(\M L)=d$, in which case $\Spc X=(\R^d,\|\cdot\|_{\M L,1})$ and $\Spc X'=(\R^d,{\|\cdot\|_{\infty, \M L^\Tr}})$
form a dual 
pair of polyhedral Banach spaces.
\end{theorem}}

\begin{proof} First, we show that $\|\M \cdot\|_{\M L,1}$ is a polyhedral norm. Using \eqref{Eq:l1infty}, we express it as $\|\M s\|_{\M L,1}=\|\M B^\Tr_N\M L \M s\|_{\ell_\infty}
=\|\M G^\Tr \M s\|_{\ell_\infty}$ with $\M G^\Tr=\M B^\Tr_N\M L \in \R^{M \times N}$, where
$\M B_N \in \R^{N \times M}$ is the binary matrix whose column vectors are the $M=2^N$ vertices of the $\ell_\infty$ unit ball (hypercube) in $N$ dimensions. 
The validity of \eqref{Eq:l1infty} for all $\M x\in \R^N$ also implies that
$\M B^\Tr_N$ has full row rank (i.e., ${\rm rank}(\M B^\Tr_N)=N$), as it would not yield a norm otherwise.
Since the left multiplication with such a matrix is rank-preserving \cite[p. 96]{Shilov1977}, one has that ${\rm rank}(\M B^\Tr_N\M L)={\rm rank}(\M L)\le d$.
We know from Theorem \ref{Theo:ConstructPoly} that  $\M s \mapsto \|\M G^\Tr \M s\|_{\ell_\infty}$
induces a polyhedral norm if and only if ${\rm rank}(\M G)=d$. Since ${\rm rank}(\M G^\Tr)={\rm rank}(\M G)$ is equal to ${\rm rank}(\M L)$ in the present setting, this is the desired result.

\revise{The dual norm is necessarily polyhedral as well. To reveal it, we consider the {\em zonotope} induced by the row vectors of $\M L$:
\begin{align}
\label{Eq:Zonotope}
Z := \sum_{n=1}^N[-\,\M u_n,\M u_n]
   = \Bigl\{\sum_{n=1}^N t_n \M u_n : \|\M t\|_{\ell_\infty}\le 1\Bigr\}.
\end{align}
This is a convex, center-symmetric subset of $\R^d$ that includes the origin as an interior point (because the $\M u_n$ span $\R^d$).
Since $Z$ is also bounded, its gauge $\mu_Z$ defines a valid norm on $\R^d$ (by Theorem \ref{Theo:GaugeNorms}). Specifically, we have that
\begin{align}
\mu_Z(\M y)
= \inf\{\lambda>0 : \M y\in \lambda Z\}
= \inf\bigl\{\|\M t\|_{\ell_\infty} : \M y=\textstyle\sum_{n=1}^N t_n\,\M u_n\bigr\},
\end{align}
with
$Z$ being the unit ball of $\mu_Z$ by construction. The dual norm of $\mu_Z$ is given by
\begin{align}
\forall \M x\in \R^d:\quad \mu_{Z^\circ}(\M x)= \sup_{\M z \inR^d: \mu_Z(\M z)\le 1}  \langle \M z, \M x\rangle = \sup_{\M z \in Z} \langle \M z, \M x\rangle
\end{align}
where $Z^\circ\subset \R^d$ is the polar of $Z$ (see Appendix A).
Now the key is that
\begin{align}
\sup_{\M z\in Z}\langle \M z,\M x\rangle
= \sup_{\|\M t\|_{\ell_\infty}\le 1}\sum_{n=1}^N t_n\langle \M u_n,\M x\rangle
= \sum_{n=1}^N |\langle \M u_n,\M x\rangle|
= \|\M L\M x\|_{\ell_1}
\end{align}
with the maximum being attained for $t_n={\rm sign}(\langle \M u_n,\M x\rangle)$.
Consequently,
we can identify $\mu_{Z^\circ}=\|\cdot\|_{\M L,1}$, which, by duality, implies that
$\mu_{Z}=\|\cdot\|_{\infty,\M L^\Tr}$. This then also yields $Z=B_{\Spc X'}$ and $Z^\circ=B_{\Spc X}$.}

\end{proof}

The matrix $\M G$ identified in the proof of Theorem \ref{Theo:L1norm} yields the explicit description of the underlying
unit ball, as formulated in the final part of Theorem \ref{Theo:ConstructPoly}; that is, the canonical polyhedral form \eqref{Eq:FacialUnitBall} with $\M F={\rm Ext}(\M G)=[\M f_1\  \cdots \ \M f_M]$. This is illustrated in Figure \ref{Fig:UnitBall}b with the facet vectors $\M f_m$ being overlaid.

The only limitation of \eqref{Eq:L1reg} with trainable $\M L$ is that this parameterization is not universal (\revise{see Proposition \ref{Prop:NonUniversal} in Appendix B}).
Yet, we can make an interesting connection with the synthesis formulation in 
\eqref{Eq:InvDico2}:
The rank condition in Theorem \ref{Theo:L1norm} guarantees the existence of a generalized inverse matrix $\widetilde{\M G}=[\tilde{\M g}_1 \cdots \tilde{\M g}_N] \in \R^{d \times N}$ (not necessarily unique) such that
\begin{align}
\label{Eq:GInverse}
\widetilde{\M G}\M L= \M I_{d},
\end{align}
with the canonical solution being the pseudo-inverse
$\widetilde{\M G}=\M L^\dagger=( \M L^\Tr \M L)^{-1} \M L^\Tr$.
By making the change of variable $\M z=\M L\M s$ in \eqref{Eq:L1reg} and by observing that $\widetilde{\M G}$ is a valid (bilateral) inverse of $\M L$ as long as its domain is restricted to ${\rm Ran}(\M L)=\{\M z=\M L \M s: \M s \in \R^d\}$, we get
\begin{align}
\label{Eq:L1reg2}
\tilde S_{\rm I}=\{\tilde{\M s}=\widetilde{\M G} \tilde{\M z}: \  \tilde{\M z} \in \arg \min_{\M z \in {\rm Ran}(\M L)} \tfrac{1}{2}\|\M y - \M H \widetilde{\M G} \M z\|_2^2 + \lambda \|\M z\|_{\ell_1}\},
\end{align}
which is almost the same as \eqref{Eq:InvDico1}, except for the restriction on the search space for $\M z$.
While the latter results in some loss of expressivity, 
it has the advantage of making the training of $\widetilde{\M G}$ much better posed by getting rid of the aforementioned ``nontrivial null space'' problem.

To demonstrate our point, we now concentrate on the scenario where the regularization operator can be factorized as $\M L=\V \Lambda \M T_{\V \theta}^\Tr$, where $\M \Lambda={\rm diag}(\lambda_1,\cdots,\lambda_N)$ is a diagonal matrix of adjustable weights with $\lambda_n>0$
and $\M T_{\V \theta} \in \R^{d \times N}$ is a tight frame parameterized by $\V \theta$. The defining property here is 
\revise{\begin{align}
\label{Eq:Tightframe}
\M T_{\V \theta} \M T_{\V \theta}^\Tr = \M I_d,
\end{align}
which makes the inversion process 
straightforward.} By using the same technique as in the derivation of
\eqref{Eq:L1reg2}, we reformulate \eqref{Eq:L1reg} as 
\begin{align}
\label{Eq:L1reg3p}
S_{\rm III}&=\{\tilde{\M s}=\M T_\V \theta \tilde{\M z}: \  \tilde{\M z} \in \arg \min_{\M z \in {\rm Ran}(\M T_\V \theta^\Tr)} \tfrac{1}{2}\|\M y - \M H \M T_{\V \theta} \M z\|_2^2 +  \|\V \Lambda \M z\|_{\ell_1}\}.
\end{align}
Next, we define the barrier functional
\begin{align}
\label{Eq:Barrier}
i_{\M T_{\V \theta}^\Tr}(\M z)=\begin{cases}
0,& \mbox{ if } \M z \in {\rm Ran}(\M T_\V \theta^\Tr)=\{\M z \in \R^N: \M T_{\V \theta}^\Tr \M T_{\V \theta} \M z=\M z\}\\
+\infty,& \mbox{otherwise}\end{cases}
\end{align}
which is lower semicontinuous and convex.
With the help of the latter, we recast \eqref{Eq:L1reg3p} as
\begin{align}
\label{Eq:L1reg3}
S_{\rm III}&=\{\tilde{\M s}=\M T_\V \theta \tilde{\M z}: \  \tilde{\M z} \in \arg \min_{\M z \in \R^N} \tfrac{1}{2}\|\M y - \M H \M T_\V \theta \M z\|_2^2 + i_{\M T_\V \theta^\Tr}(\M z)+  \|\V \Lambda \M z\|_{\ell_1}\}
\end{align}
which is easier to minimize 
because the constraint on the search space has been removed.

\subsection{Implementation of a Trainable Convex Regularizer}
A standard requirement in computational imaging is that the regularizer be shift-invariant. In the context of \eqref{Eq:L1reg}, this is achieved by employing an $N_{\rm chan}$-channel filterbank as the operator $\M L$. This takes $\M s \in \R^d$ as input and produces $N_{\rm chan}$ output channels, resulting in a feature vector $\M z=\M L \M s$ of size $N=N_{\rm chan} \times d$. In our implementation, we use a Parseval filterbank
$\M T^\Tr_\V \theta$ (tight frame) parametrized by an orthogonal matrix of size $N_{\rm chan} \times N_{\rm chan}$, as described in \cite{Unser2025a}.
We also found it beneficial to force all the active regularization filters to have zero mean. This constraint is imposed within our orthogonal parametrization by setting the first convolution mask to be proportional to $\M 1$, effectively creating a moving average filter.

To learn the regularizer $\M s \mapsto \|\V \Lambda\M T^\Tr_{\V \theta} \M s\|_{\ell_1}$ that  best represents a given class of signals/images, we follow the strategy of \cite{Goujon2023} and consider a basic denoising task with $\M H=\M I$ where the signal is corrupted by additive white Gaussian noise. To adjust the underlying model such as to achieve the best denoising on a representative set of images, we unroll the recurrent neural network in Algorithm 1 and/or use deep equilibrium \cite{Gilton2021} to learn the model parameters $(\V \Lambda,\V \theta)$.

For the purpose of experimentation, we considered the more general family of optimization problems
\begin{align}
\label{Eq:TrainNonLin}
S_{\rm IV}&=\arg \min_{\M s \in \R^d} \tfrac{1}{2}\|\M y - \M H \M s\|_2^2 +   \langle \M 1, \V \Phi(\M T^\Tr_{\V \theta}\M s)\rangle\\
&=\big\{\tilde{\M s}=\M T_\V \theta \tilde{\M z}: \  \tilde{\M z} \in \arg \min_{\M z \in \R^N} \tfrac{1}{2}\|\M y - \M H \M T_\V \theta \M z\|_2^2 + i_{\M T_\V \theta^\Tr}(\M z)+  \langle \M 1, \V \Phi(\M z)\rangle \big\}
\label{Eq:AnabySyn}
\end{align}
where $\V \Phi=(\phi_n): \R^N \to \R^N$ is a generic (pointwise) potential with convex component functions $\phi_n: \R \to \R_{\ge 0}$. Our implementation is based on the equivalent analysis-by-synthesis formulation \eqref{Eq:AnabySyn}, which decouples the effect of the filters from the nonlinearities induced by $\V \Phi$. The corresponding regularizer is
\begin{align}
\label{Eq:SepReg}
g(\M z)=\langle \M 1, \V \Phi(\M z)\rangle=\sum_{n=1}^N \phi_n(z_n),
\end{align}
which represents the most general separable form of  potential.
Equation \eqref{Eq:L1reg3} is recovered by setting $\phi_n(z)=\lambda_n|z|$.
Here, we also consider the option of learning the $\phi_n$ using the framework described in \cite{Unser2025ACHA},
which yields an architecture that offers greater expressivity than \eqref{Eq:L1reg3}.

Since \eqref{Eq:AnabySyn} is the sum of three functionals
whose gradient and/or individual proximal maps are easy to compute, it can be minimized
efficiently using proximal splitting methods \cite{Condat2023}.
For our experiments, we adopted the Douglas-Ratchford splitting \cite{BricenoArias2013,Raguet2018}, which resulted in the application-specific implementation 
summarized in Algorithm 1.

\begin{algorithm}
\begin{algorithmic}
\vspace*{0.5ex}
   \State \textit{Input:} The data vector $\M y \in \R^M$, the proximal operator of the potential $\V \Phi$, the tight frame $\M T_{\V \theta}$, and the step size $\tau$.
	\vspace{.5em}
   \State \textit{Initialization:} 
   \State \quad $\M z^{(0)}=\M T_{\V \theta}^\Tr \M H^\Tr \M y$\quad  (backprojection)
	\vspace{.5em}
   \State \textit{Main loop:} 
   \For{$n=0,1,2,\ldots$}
   \State $\M z^{(n+\tfrac{1}{2})} =  \;{\rm prox}_{\tau \V \Phi}\left(\M z^{(n)} - \tau \M T_{\theta}^\Tr \M H^\Tr (\M H \M T_{\theta} \M z^{(n)} - \M y)\right)$
   \State  $\M z^{(n+1)}= \M T_{\theta}^\Tr \M T_{\theta}\left(2\M z^{(n + \tfrac{1}{2})} - \M z^{(n)}\right) + \M z^{(n)} - \M z^{(n + \tfrac{1}{2})}$
  \EndFor
	\vspace{.5em}
   \State \textit{Output:} ${\tilde{ \M s}}=\M T_{\V \theta} \M z^{(n)}$ as solution of  \eqref{Eq:TrainNonLin}.
\end{algorithmic}
\caption{Douglas Ratchford splitting algorithm (DRS) for the resolution of inverse problem with \eqref{Eq:TrainNonLin} when the potential $\V \Phi$ is separable with known (or trainable) proximal operator ${\rm prox}_{\V \Phi}$.
\vspace*{0.5ex}}
\label{Algo:drs}\end{algorithm}

Besides the gradient of the data term, the key components for Algorithm~\ref{Algo:drs} are the proximal maps for: (i) the barrier functional $i_{\M T_{\V \theta}^\Tr}(\M z)$, and (ii) the separable regularizer $g$.
The proximal operator of the first corresponds to the projection operator 
$\M T_{\V \theta}^\Tr\M T_{\V \theta}: \R^{N} \to {\rm Ran}(\M T_{\V \theta}^\Tr)$, which is also used in \eqref{Eq:Barrier} to indicate the range constraint.
The proximal operator for $\|\M \Lambda \M z\|_{\ell_1}$ is the soft-threshold with parameter $\M \Lambda$. 
The proximal operator for the more general regularizer \eqref {Eq:SepReg} is separable as well, and given
by ${\rm prox}_{g}=({\rm prox}_{\phi_n})$ where each component is defined by \begin{align}
{\rm prox}_{\phi_n}(x)=\arg \min_{z \inR} \tfrac{1}{2}(x-z)^2 + \phi_n(z).
\end{align}
These functions must be monotone and firmly non-expansive \cite{Gribonval2020}, and can be represented and learned effectively using linear splines, as proposed in \cite{Unser2025ACHA}. 

\begin{algorithm}
\begin{algorithmic}
\vspace*{0.5ex}
   \State \textit{Input:} The data vector $\M y \in \R^M$, the channel-wise derivatives $\V \varphi=(\phi'_n)$ of the potential functions $\V \Phi=(\phi_n)$, the tight frame $\M T_{\V \theta}$, and the step size $\tau$.
	\vspace{.5em}
   \State \textit{Initialization:} 
   \State \quad $\M z^{(0)}=\M T_{\V \theta}^\Tr \M H^\Tr \M y$\quad  (backprojection)
	\vspace{.5em}
   \State \textit{Main loop:} 
   \For{$n=0,1,2,\ldots$}
   \State $\M z^{(n+\tfrac{1}{2})} = \M z^{(n)} - \tau \left(\M T_{\theta}^\Tr \M H^\Tr (\M H \M T_{\theta} \M z^{(n)} - \M y) + \V \varphi(\M z^{(n)})\right)$
   \State  $\M z^{(n+1)}= \M T_{\theta}^\Tr \M T_{\theta} \M z^{(n + \tfrac{1}{2})}$
  \EndFor
	\vspace{.5em}
   \State \textit{Output:} ${\tilde{ \M s}}=\M T_{\V \theta} \M z^{(n)}$ as solution of  \eqref{Eq:TrainNonLin}.
\end{algorithmic}
\caption{Accelerated proximal gradient descent algorithm (APGD) for the resolution of inverse problem \eqref{Eq:TrainNonLin} when the potential $\V \Phi$ is separable with trainable gradient $\V \phi=(\phi_n')$.\vspace*{0.5ex}}
\label{Algo:apgd}\end{algorithm}

For completeness, and to establish a connection with the best-performing convex FoE models \cite{Goujon2023}, we also implemented a variant of the reconstruction scheme (Algorithm~\ref{Algo:apgd}) for differentiable potentials $\V \Phi$. This variant makes use of the gradient $\nabla \V \Phi_{\M z} = (\phi_n')$, which can likewise be learned and parameterized using linear splines. \revise{The detailed description and justification of this scheme can be found in \cite{Unser2025ACHA}}.

\revise{Since the functionals to minimize are convex, Algorithms \ref{Algo:drs} and \ref{Algo:apgd} both converge to a global minimum whenever $\tau < 2/\rho$ where $\rho$ is the operator norm of $\M T_\theta^\top \mathbf H^\top \mathbf H \mathbf T_\theta$ \cite{Bauschke2011convex}. As long as this condition is fulfilled, $\tau$ has very little incidence on performance.}

\subsection{Experimental results}

For validation purposes, we applied our framework to the denoising of natural images and to the reconstruction of magnetic resonance data.

\subsubsection{Training}
The training was performed on a basic denoising task and is common to \revise{all scenarios, including MRI reconstruction}.
The dataset consists of 238400 patches of size $(40 \times 40)$ taken from the BSD500 image dataset \cite{Arbelaez2011}.
All noise-free images $\M s \in \R^d$ are normalized to take values in $[0, 1]$. They were then corrupted with additive Gaussian noise of standard deviation $\sigma$ to yield the data $\M y \in \R^d$. 

Our convex denoisers come in three variants: The first, denoted by $\M f_{\V \theta,\V \Lambda}$, is specified by \eqref{Eq:L1reg3} with $\M H=\M I$.
The second and third are both specified by \eqref{Eq:AnabySyn} with $\M H=\M I$, but differ
in the training strategy applied to $\V \Phi$; namely, proximal (Algorithm 1) versus gradient-based (Algorithm 2). Since $\V \Phi$ is trainable through its proximity operator, the second architecture is {\em a priori} more expressive than the first, which is restricted to polyhedral regularization, in accordance with the theory.
All denoisers use Parseval filterbanks of size $W \times W$ as regularization operators, with a total number of channel $N_{\rm chan}=W^2$.
\revise{These filters are specified by an orthogonal matrix $\M U \in  \R^{N_{\rm chan}\times N_{\rm chan}}$, which is itself parameterized using the Bj{\"o}rck algorithm \cite{Bjorck1971}.}
They are trained in PyTorch for the regression task $\M f_{\V \theta,\V \Lambda}(\M y)\approx \M s$. \revise{Since the Bj{\"o}rck parametrization is exact,  the tight-frame relation \eqref{Eq:Tightframe} is guaranteed to hold at every stage of the optimization process.}
During training, the algorithm was run until the relative difference of the iterates fell below 1e-4. Gradients were estimated via the deep equilibrium framework using the Broyden algorithm with 25 iterates \cite{Broyden1965}. 

\revise{Since $N_{\rm chan}=W^2$ and the complexity of each filter is $O(W^2)$,  the computational cost for a full training grows quadratically with the number of channels. The same holds true for inference which then reduces to a ``standard'' iterative resolution of a convex problem.}

\subsubsection{Denoising results}
\label{Sec:Denoising}
\begin{table}[tb]
\centering
\label{table:denoising_perf}
\begin{tabular}{lccccc}
\hline
Filter Size &2$\times$1& 3$\times$3 & 5$\times$5 & 7$\times$7 & 9$\times$9 \\
\hline
Total Variation & 27.48 \\
Weighted $\ell_1$ (polyhedral norm)&& {\bf 27.86} & {\bf 27.98} & 28.01 & 28.03 \\
Learned Prox && 27.80 & {\bf 27.98} & {\bf 28.02} & {\bf 28.04} \\
Learned Gradient && 27.79 & 27.95 & 27.99 & 28.01 \\
\hline
\end{tabular}
\caption{PSNR (in dB) on BSD68 with noise level $\sigma=25/255$. \label{Tab:Denoise}}
\vspace{.2cm}
\end{table}

The denoising performance on the BSD68 test set is reported in Table~\ref{Tab:Denoise} for $\sigma =  25/255$.
All trained denoisers outperform total variation (TV) denoising, which is included as a baseline.
We observe that their performance improves with the number of regularization channels, up to a point where it saturates.
Remarkably, the weighted $\ell_1$ scheme performs on par with “Learned Prox,” despite the latter having significantly more degrees of freedom and the capacity to learn arbitrary pointwise proximal nonlinearities. Even more striking is the fact that “Learned Prox” converged to nonlinearities that closely resemble soft-threshold functions, as shown in Figure~\ref{Fig:Prox}.
Another general trend is that the learned weight $\lambda_1$ for the first channel (corresponding to the moving average) was consistently close to zero, effectively preserving the signal’s lowpass content---see identity-like graph on upper left in Figure~\ref{Fig:Prox}.

We attribute the slightly lower performance of “Learned Gradient” to its difficulty to learn a polyhedral norm which originates from a potential that is non-differentiable at the origin---an observation that provides additional supports for our universality result.

\begin{figure}[tpb]
    \centering
  \includegraphics[scale=0.5]{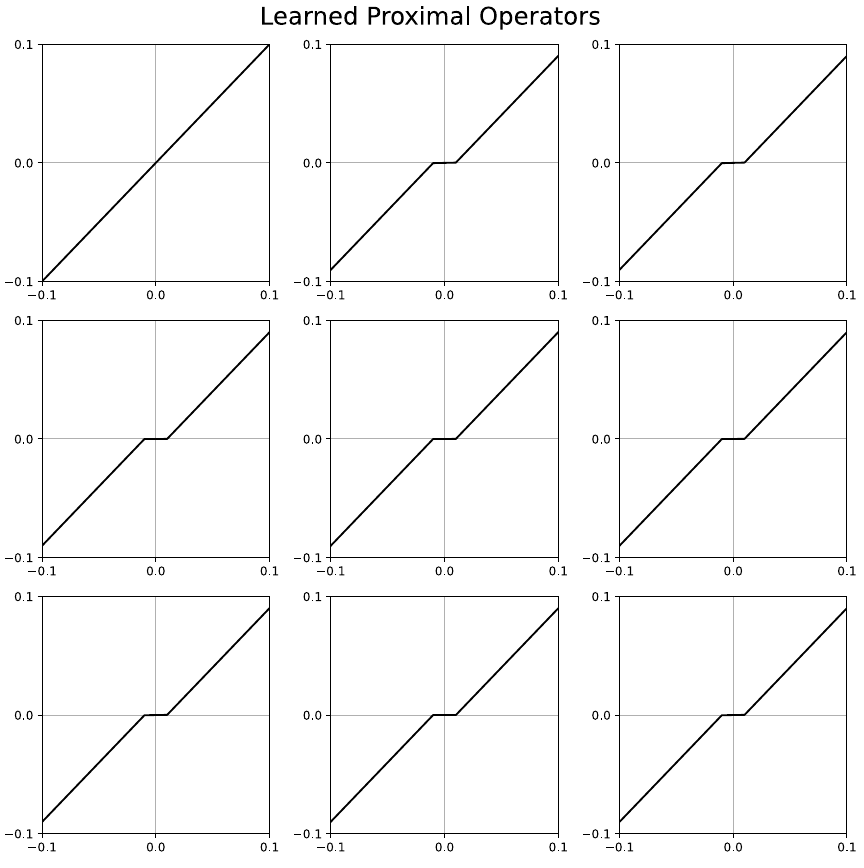}
    \caption{Learned proximal nonlinearities in the $3 \times 3=9$ channel scenario. \label{Fig:Prox}
}
\end{figure}
\revise{\subsubsection{Learning Regularizers through Denoising}
In the traditional variational or Bayesian formulation of inverse problems, the regularization functional 
$g: \R^d \to \R$ in \eqref{Eq:VarRecontruct} encodes the prior information on the signal and is, in principle, independent of the specific reconstruction task. When we train the system for denoising, as in Section~\ref{Sec:Denoising}, we effectively fix the regularizer to $g=g_{\V \theta}$, which leaves one remaining degree of freedom: the regularization parameter $\lambda$.}

\revise{To assess how well our weighted $\ell_1$ regularization complies with this classical paradigm, we conducted additional denoising experiments across a range of noise levels and three distinct training strategies. In the first scenario (retraining), the regularizer is retrained separately for each noise level. In the second scenario (no retraining), the system is trained once at a fixed noise level $\sigma=15/255$ and then applied to all other noise levels, with the only adaptation being the tuning of $\lambda$.
In the third scenario (joint training), a single model is trained on the entire range $\sigma \in [0, 0.1]$ and then used for inference in the same way as in Scenario 2. In both Scenarios 2 and 3, the value of $\lambda$ is adjusted automatically through a lightweight neural network that predicts the optimal parameter $\lambda=\lambda(\sigma)$.}

\revise{The performance of the three approaches at noise levels $\{5/255, 10/255,$ $15/255,$ $20/255, 25/255\}$
is summarized in Table~\ref{tab:robust}. The results are nearly identical across all strategies, indicating that our method generalizes well to varying noise conditions. This strongly supports the adequacy of the variational formulation \eqref{Eq:VarRecontruct} with a fixed, universal regularizer 
$g=g_\theta$ whose strength is modulated by a single hyperparameter $\lambda$.}
\revise{\begin{table}[]
\begin{tabular}{llllll}
\hline\hline
 & $\sigma=\frac{5}{255}$ & $\sigma=\frac{10}{255}$ & $\sigma=\frac{15}{255}$ & $\sigma=\frac{20}{255}$ & $\sigma=\frac{25}{255}$ \\[0.5ex]
\hline
1. With retraining & 36.91 & 32.67 & \bf{30.48} & 29.06 & 28.03 \\
2. Without retraining & 36.89 & 32.63 & \bf{30.48} & 29.06 & 28.02 \\
3. Joint training & 36.91 & 32.67 & \bf{30.48} & 29.04 & 28.02 \\
\hline\hline
\end{tabular}
\caption{\label{tab:robust}
Comparison of denoising performance (PSNR in dB) across noise levels and training strategies.
}
\end{table}
}

\subsubsection{MRI Reconstruction}
For the MRI reconstruction experiments, we used Parseval filters of size $7 \times 7$. \revise{We follow the same protocol as in \cite{Unser2025a}
where the image is reconstructed based on the minimization of \eqref{Eq:VarRecontruct} with a fixed pretrained regularization functional  $g=g_{\V \theta}$ and a single hyperparameter $\lambda \in \R_{\ge0}$ that is optimized for best performance. 
For each regularization scenario, the training of $g_{\V \theta}$ is achieved by denoising according to the procedure described in Section 
\ref{Sec:Denoising}.}
We investigated three Fourier-domain sampling schemes: (a) random; (b) radial; and (c) (semi-)Cartesian with non-uniform sampling along the horizontal direction—each corresponding to a specific system matrix $\M H$. The reconstruction algorithm was run until the relative difference between successive iterates dropped below $10^{-5}$.

The reconstruction performance across the different acquisition protocols and regularization strategies is summarized in Table~\ref{tab:mr_reconstruction} for two representative scans (brain and bust). The overall trend mirrors that observed in the denoising experiments: notably, that our \revise{pretrained} polyhedral regularizer performs on par with the two alternative schemes.

\begin{table}[!tb]
\begin{tabular}{l|cccccc}
\hline\hline \text { Subsampling mask } & \multicolumn{2}{c}{\text { Random }} & \multicolumn{2}{c}{\text { Radial }} & \multicolumn{2}{c}{\text { Cartesian }} \\
\text { Image type } & \text { Brain } & \text { Bust } & \text { Brain } & \text { Bust } & \text { Brain } & \text { Bust } \\
\hline \text { Zero-filling } & 23.18 & 24.90 & 22.57 & 23.51 & 20.85 & 22.34\\
\text { TV } & 26.49 & 28.40 & 25.71 & 27.84 & 22.85 & 25.57\\
\text{ Weighted} $\ell_1$ & \textbf{27.48} & \textbf{29.07} & 26.69 & 28.53 & 23.13 & 25.83\\
\text{ Learned Prox} & \textbf{27.48} & \textbf{29.07} & \textbf{26.70} & \textbf{28.56} & \textbf{23.14} & \textbf{25.84}\\
\text{ Learned Grad} & 27.47 & 29.05 & \textbf{26.70} & 28.54 & 23.12 & 25.81\\
\hline\hline
\end{tabular}
\caption{PSNR values (in dB) for MRI reconstruction.}
\label{tab:mr_reconstruction}
\end{table}

The results for the Brain dataset with radial Fourier sampling are shown in Figures~\ref{Fig_MRI}. In the lower panel, the zero-filled reconstruction exhibits noticeable overlaid textured noise. This reconstruction noise is substantially reduced using total variation (TV)—a method commonly employed for this purpose—and even further suppressed by the three learned regularizers, all of which produce visually similar outputs.
\begin{figure}[tpb]
    \centering
  \includegraphics[scale=0.6]{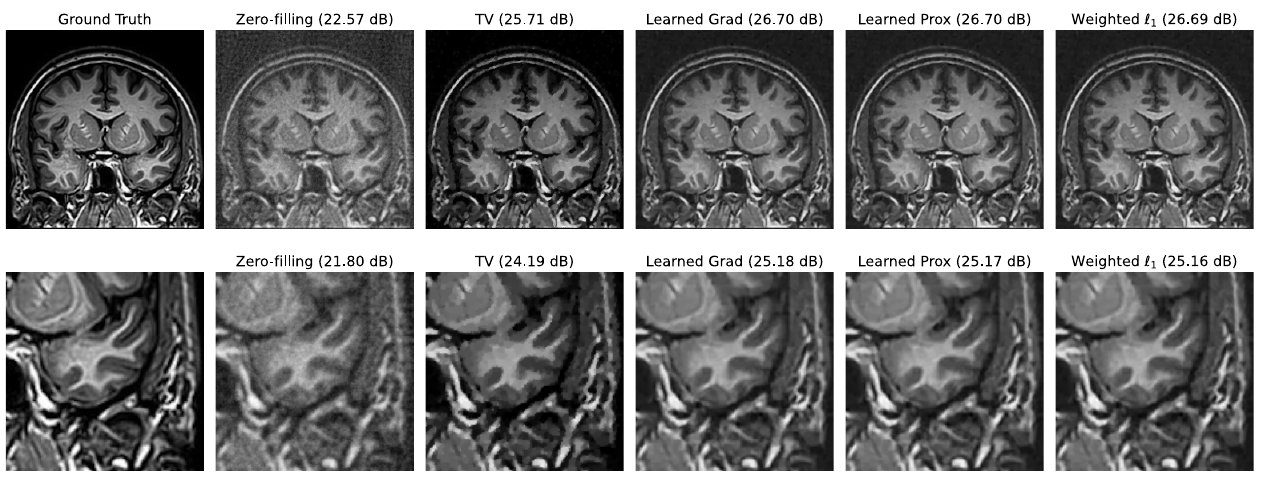}
    \caption{
 Comparison of ground truth, zero-fill/backprojection $\M H^{\Tr} \M y$, and four variational reconstructions of the brain image from radially sampled Fourier data.
 Lower panel: zoom of a region of interest. The SNR is evaluated with respect to the groundtruth.
    \label{Fig_MRI}
}
\end{figure}
\subsection{Discussion}
Our experiments demonstrate that our weighted-$\ell_1$ scheme matches the performance of state-of-the-art ML frameworks for convex image reconstruction, while requiring significantly fewer parameters and reduced computational resources.

That said, there remains room for improvement, as the architecture used in our experiments only partially exploits the degrees of freedom offered by the polyhedral framework. This limitation arises because of: (i) our use of Parseval filterbanks, rather than general convolutional operators, to simplify training; and (ii) the adoption of an analysis-by-synthesis formulation, which serves only as a proxy for the universal dictionary-based architecture established by our theory.

Further research is needed to develop an architecture that is both machine-learning friendly and universal in the sense described in this work. Such an architecture would help delineate the ultimate capabilities of convex regularization.

\pagebreak
\section*{Appendix A: (Semi)Norms and Banach Duality}
\addcontentsline{toc}{section}{Appendix A: (Semi)norms and Banach Duality}
\begin{definition}[Norms and seminorms]
\label{Def:norm1}
A seminorm $p$ on a vector space $\Spc X$ is a real-valued functional $\Spc X \to \R_{\ge 0}$ 
that, for all $x \in \Spc X$, satisfies the following properties: \begin{itemize}
\item Homogeneity: $|\alpha| \,p(x)=p(\alpha \,x)$ for all $\alpha \in \R$ (or $\C$).
\item Subadditivity (or triangle inequality): $
p(x+y)\le p(x)+p(y)$ for all $x,y \in \Spc X$.
\end{itemize}
If, in addition, $p(x)=0 \Leftrightarrow x=0$ (positive-definiteness), then $p$ is called norm and is
 commonly denoted by $\|\cdot\|=p(\cdot)$.
 
A {\em normed space} is a vector space equipped with a norm---the generic form being $(\Spc X,\|\cdot\|_\Spc X)$.
The latter is a {\em Banach space} whenever it is complete; that is, when every Cauchy sequence in $(\Spc X,\|\cdot\|_{\Spc X})$ has a limit in $\Spc X$.
\end{definition}

We denote the (closed) unit ball of a Banach space $\Spc X=(\Spc X,\|\cdot\|_{\Spc X})$ by $B_\Spc X$ with
\begin{align}
B_{\Spc X}=\{x \in \Spc X: \|x\|_{\Spc X}\le 1\}.
\end{align}
The continuous dual of $\Spc X$ is the vector space $\Spc X'$ of continuous linear functionals on $\Spc X$ with generic element $f: x \mapsto \langle f, x\rangle \in \R$. A fundamental property is that
$\Spc X'=(\Spc X',\|\cdot\|_{\Spc X'})$ 
 equipped with the dual norm
\begin{align}
\label{Eq:DualNorm}
\forall f \in \Spc X': \quad \|f\|_{\Spc X'}=\sup_{x \in B_{\Spc X}}\langle f, x\rangle.
\end{align}
is a Banach space as well \cite{Megginson2012introduction}.

The unit ball $B_{\Spc X}$ is a convex, center-symmetric subset
of $\Spc X$ (i.e., a disk) whose rescaled versions specify all Banach neighborhoods of the origin
$0 \in \Spc X$.
Similar to \eqref{Eq:DualNorm} which makes an explicit connection between $B_{\Spc X}$ and the dual norm, 
it possible to relate $B_{\Spc X}$ to the initial norm $\|\cdot\|_{\Spc X}=\mu_{B_{\Spc X}}(\cdot)$ with the help of the Minkowski functional (see Theorem \ref{Theo:GaugeNorms}).

In finite dimensions, one defines the polar of any subset $B\subset \R^d$  as
\begin{align}
\label{Eq:Polar}
B^\circ = \{\M y \in \R^d: \langle \M y, \M x \rangle \le 1, \M x \in B\}= \{\M y \in \R^d: \sup_{\M x \in B}\langle \M y, \M x \rangle \le 1\},
\end{align}
with the latter necessarily being a closed convex subset of $\R^d$ that contains the origin (see  \cite{Brondsted2012Polytopes}). If $B=B_{\Spc X}$ is the unit ball of the Banach space
$\Spc X=(\R^d,\|\cdot\|_{\Spc X})$, then its polar $B^\circ_{\Spc X}$ coincides with the unit ball $B_{\Spc X'}$ of the dual space, in conformity with \eqref{Eq:DualNorm}. Accordingly, we have
that $\Spc X'=\big(\R^d,\mu_{B^\circ_{\Spc X}}(\cdot)\big)=(\R^d,\|\cdot\|_{\Spc X'})$, as well as $(B^{\circ}_{\Spc X})^\circ=B_{\Spc X}$ because $\Spc X$ is reflexive.
\revise{\section*{Appendix B: Weighted $\ell_1$-norms are not universal}
\addcontentsline{toc}{section}{Appendix B: Weighted $\ell_1$ norms are not universal}
The following result implies
that the family of weighted $\ell_1$-norms is not universal in the sense of Theorem \ref{Theo:UniversalApprox}.
\begin{proposition}
\label{Prop:NonUniversal}
For $d>2$, there does not exist a rectangular matrix $\M L$ such that $\|\M x\|_{\ell_\infty}=\|\M L \M x\|_{\ell_1}$ for all $\M x \in \R^d$. \end{proposition}
\begin{proof}
Let us assume that there exists a matrix $\M L=[\M u_1 \cdots \M u_N]^\Tr$ such that $\|\M x\|_{\ell_\infty}=\|\M L \M x\|_{\ell_1}$. Then, by Theorem \ref{Theo:L1norm}, the
dual unit ball $B_{\ell_1}=B^\circ_{\ell_\infty}$ must be a zonotope $Z$ whose generic form is given by \eqref{Eq:Zonotope} \cite{McMullen1971,Shephard1974}. 
On the one hand, we know that $B_{\ell_1}$ is the $d$-dimensional cross-polytope, 
which is simplicial, with all its 2-faces being triangular. On the other hand, we can invoke
a classical result \cite[Thm. 3.5.2 p. 191]{Schneider2013}, which states that a polytope is a zonotope if and only if all its 2-faces are centrally symmetric. 
In fact, there is an even stronger result by McMullen stating that if the $k$-faces for $2\le k \le d-2$ of a $d$-polytope $P$ are centrally symmetric, then all the faces of $P$ of every dimension are centrally symmetric \cite{McMullen1970}. Ultimately this means that a polytope is a zonotope if and only if all its faces are centrally symmetric.
This leads to a contradiction for $d\ge3$ since a non-degenerate triangle cannot be centrally symmetric.
\end{proof}
}

\subsection*{Acknowlegdments}
The research was supported by the European Commission under Grant ERC-2020-AdG 
FunLearn-101020573.

\bibliographystyle{siam}


%
%

\bibliography{/Users/unser/MyDrive/Bibliography/Bibtex_files/Unser.bib}

\begin{thebibliography}{10}

\bibitem{AgHeMaJaco2019}
{\sc H.~K. Aggarwal, M.~P. Mani, and M.~Jacob}, {\em {MoDL}: {M}odel-based deep
  learning architecture for inverse problems}, IEEE Transactions on Medical
  Imaging, 38 (2019), pp.~394--405.

\bibitem{Amos2017}
{\sc B.~Amos, L.~Xu, and J.~Z. Kolter}, {\em Input convex neural networks}, in
  Proceedings of the 34th International Conference on Machine Learning,
  D.~Precup and Y.~W. Teh, eds., vol.~70 of Proceedings of Machine Learning
  Research, PMLR, 06--11 Aug 2017, pp.~146--155.

\bibitem{Antun2020}
{\sc V.~Antun, F.~Renna, C.~Poon, B.~Adcock, and A.~C. Hansen}, {\em On
  instabilities of deep learning in image reconstruction and the potential
  costs of {AI}}, Proceedings of the National Academy of Sciences, 117 (2020),
  pp.~30088--30095.

\bibitem{Arbelaez2011}
{\sc P.~Arbel{\'a}ez, M.~Maire, C.~Fowlkes, and J.~Malik}, {\em Contour
  detection and hierarchical image segmentation}, IEEE Transactions on Pattern
  Analysis and Machine Intelligence, 33 (2011), pp.~898--916.

\bibitem{Bauschke2011convex}
{\sc H.~H. Bauschke and P.~L. Combettes}, {\em Convex Analysis and Monotone
  Operator Theory in Hilbert Spaces}, CMS Books in Mathematics, Springer, New
  York, 2011.

\bibitem{Bauschke2012firmly}
{\sc H.~H. Bauschke, S.~M. Moffat, and X.~Wang}, {\em Firmly nonexpansive
  mappings and maximally monotone operators: {C}orrespondence and duality},
  Set-Valued and Variational Analysis, 20 (2012), pp.~131--153.

\bibitem{Beck2009}
{\sc A.~Beck and M.~Teboulle}, {\em Fast gradient-based algorithms for
  constrained total variation image denoising and deblurring problems}, IEEE
  Transactions on Image Processing, 18 (2009), pp.~2419--2434.

\bibitem{Bjorck1971}
{\sc A.~Bj{\"o}rck and C.~Bowie}, {\em An iterative algorithm for computing the
  best estimate of an orthogonal matrix}, SIAM Journal on Numerical Analysis, 8
  (1971), pp.~358--364.

\bibitem{Bonsall1991}
{\sc F.~Bonsall}, {\em A general atomic decomposition theorem and {B}anach's
  closed range theorem}, The Quarterly Journal of Mathematics, 42 (1991),
  pp.~9--14.

\bibitem{Boyd2004convex}
{\sc S.~Boyd and L.~Vandenberghe}, {\em Convex Optimization}, Cambridge
  University Press, 2004.

\bibitem{BricenoArias2013}
{\sc L.~M. Briceno-Arias}, {\em Forward-{D}ouglas/{R}achford splitting and
  forward-partial inverse method for solving monotone inclusions},
  Optimization, 64 (2013), pp.~1239--1261.

\bibitem{Brondsted2012Polytopes}
{\sc A.~Brondsted}, {\em An Introduction to Convex Polytopes}, vol.~90,
  Springer Science \& Business Media, 2012.

\bibitem{Bronshteyn1975}
{\sc E.~M. Bronshteyn and L.~Ivanov}, {\em The approximation of convex sets by
  polyhedra}, Siberian Mathematical Journal, 16 (1975), pp.~852--853.

\bibitem{Bronstein2008}
{\sc E.~M. Bronstein}, {\em Approximation of convex sets by polytopes}, Journal
  of Mathematical Sciences, 153 (2008), pp.~727--762.

\bibitem{Broyden1965}
{\sc C.~G. Broyden}, {\em A class of methods for solving nonlinear simultaneous
  equations}, Mathematics of Computation, 19 (1965), pp.~577--593.

\bibitem{Candes2011compressed}
{\sc E.~Cand{\`{e}}s, Y.~Eldar, D.~Needell, and P.~Randall}, {\em Compressed
  sensing with coherent and redundant dictionaries}, Applied and Computational
  Harmonic Analysis, 31 (2011), pp.~59--73.

\bibitem{Candes2007}
{\sc E.~J. Cand{\`e}s and J.~Romberg}, {\em Sparsity and incoherence in
  compressive sampling}, Inverse Problems, 23 (2007), pp.~969--985.

\bibitem{Carrera2017}
{\sc D.~Carrera, G.~Boracchi, A.~Foi, and B.~Wohlberg}, {\em Sparse
  overcomplete denoising: Aggregation versus global optimization}, IEEE Signal
  Processing Letters, 24 (2017), pp.~1468--1472.

\bibitem{Chan2016plug}
{\sc S.~H. Chan, X.~Wang, and O.~A. Elgendy}, {\em Plug-and-play {ADMM} for
  image restoration: Fixed-point convergence and applications}, IEEE
  Transactions on Computational Imaging, 3 (2016), pp.~84--98.

\bibitem{Chandrasekaran2012}
{\sc V.~Chandrasekaran, B.~Recht, P.~A. Parrilo, and A.~S. Willsky}, {\em The
  convex geometry of linear inverse problems}, Foundations of Computational
  mathematics, 12 (2012), pp.~805--849.

\bibitem{Chen2017}
{\sc Y.~Chen and T.~Pock}, {\em Trainable nonlinear reaction diffusion: A
  flexible framework for fast and effective image restoration}, IEEE
  Transactions on Pattern Analysis and Machine Intelligence, 39 (2017),
  pp.~1256--1272.

\bibitem{Chen2014Learning}
{\sc Y.~Chen, T.~Pock, and H.~Bischof}, {\em Learning {$\ell_1$}-based analysis
  and synthesis sparsity priors using bi-level optimization}, in Workshop on
  Analysis Operator Learning vs. Dictionary Learning, NIPS 2012, Jan. 2014.

\bibitem{Christensen2003}
{\sc O.~Christensen}, {\em An Introduction to Frames and {R}iesz Bases},
  Birkhauser, 2003.

\bibitem{Combettes2005}
{\sc P.~Combettes and V.~Wajs}, {\em Signal recovery by proximal
  forward-backward splitting}, Multiscale Modeling $\&$ Simulation, 4 (2005),
  pp.~1168--1200.

\bibitem{Condat2023}
{\sc L.~Condat, D.~Kitahara, A.~Contreras, and A.~Hirabayashi}, {\em Proximal
  splitting algorithms for convex optimization: A tour of recent advances, with
  new twists}, SIAM Review, 65 (2023), pp.~375--435.

\bibitem{delAguilaPla2023}
{\sc P.~del Aguila~Pla, S.~Neumayer, and M.~Unser}, {\em Stability of
  image-reconstruction algorithms}, {IEEE} Transactions on Computational
  Imaging, 9 (2023), pp.~1--12.

\bibitem{Donoho2006}
{\sc D.~L. Donoho}, {\em Compressed sensing}, IEEE Transactions on Information
  Theory, 52 (2006), pp.~1289--1306.

\bibitem{Ekeland1999}
{\sc I.~Ekeland and R.~Temam}, {\em Convex Analysis and Variational Problems},
  vol.~28 of Classics in Applied Mathematics, {SIAM}, 1999.

\bibitem{Elad2010b}
{\sc M.~Elad}, {\em Sparse and Redundant Representations. From Theory to
  Applications in Signal and Image Processing}, Springer, 2010.

\bibitem{Figueiredo2007}
{\sc M.~A. Figueiredo, R.~D. Nowak, and S.~J. Wright}, {\em Gradient projection
  for sparse reconstruction: Application to compressed sensing and other
  inverse problems}, IEEE Journal of Selected Topics in Signal Processing, 1
  (2007), pp.~586--597.

\bibitem{Figueiredo2003}
{\sc M.~A.~T. Figueiredo and R.~D. Nowak}, {\em An {EM} algorithm for
  wavelet-based image restoration}, IEEE Transactions on Image Processing, 12
  (2003), pp.~906--916.

\bibitem{Foucart2013}
{\sc S.~Foucart and H.~Rauhut}, {\em A Mathematical Introduction to Compressive
  Sensing}, Springer, 2013.

\bibitem{Gilton2021}
{\sc D.~Gilton, G.~Ongie, and R.~Willett}, {\em Deep equilibrium architectures
  for inverse problems in imaging}, IEEE Transactions on Computational Imaging,
  7 (2021), pp.~1123--1133.

\bibitem{Goujon2023}
{\sc A.~Goujon, S.~Neumayer, P.~Bohra, S.~Ducotterd, and M.~Unser}, {\em A
  neural-network-based convex regularizer for inverse problems}, IEEE
  Transactions on Computational Imaging, 9 (2023), pp.~781--795.

\bibitem{Goujon2024Weakly}
{\sc A.~Goujon, S.~Neumayer, and M.~Unser}, {\em Learning weakly convex
  regularizers for convergent image-reconstruction algorithms}, {SIAM} Journal
  on Imaging Sciences, 17 (2024), pp.~91--115.

\bibitem{Gribonval2020}
{\sc R.~Gribonval and M.~Nikolova}, {\em A characterization of proximity
  operators}, Journal of Mathematical Imaging and Vision, 62 (2020),
  pp.~773--789.

\bibitem{GRUBER1993}
{\sc P.~M. Gruber}, {\em Aspects of Approximation of Convex Bodies}, Elsevier,
  1993, pp.~319--345.

\bibitem{Gruber1998}
{\sc P.~M. Gruber}, {\em Asymptotic estimates for best and stepwise
  approximation of convex bodies {IV}}, Forum Mathematicum, 10 (1998),
  pp.~665--686.

\bibitem{Gruber1982}
{\sc P.~M. Gruber and P.~Kenderov}, {\em Approximation of convex bodies by
  polytopes}, Rendiconti del Circolo Matematico di Palermo, 31 (1982),
  pp.~195--225.

\bibitem{Grunbaum2000Polytope}
{\sc B.~Gr{\"u}nbaum}, {\em Convex Polytopes}, Springer, 2nd~ed., 2000.

\bibitem{Hammernik2018}
{\sc K.~Hammernik, T.~Klatzer, E.~Kobler, M.~P. Recht, D.~K. Sodickson,
  T.~Pock, and F.~Knoll}, {\em Learning a variational network for
  reconstruction of accelerated {MRI} data}, Magnetic Resonance in Medicine, 79
  (2018), pp.~3055--3071.

\bibitem{Hertrich2021}
{\sc J.~Hertrich, S.~Neumayer, and G.~Steidl}, {\em Convolutional proximal
  neural networks and {P}lug-and-{P}lay algorithms}, Linear Algebra and Its
  Applications, 631 (2021), pp.~203--234.

\bibitem{Jin2017}
{\sc K.~H. Jin, M.~T. McCann, E.~Froustey, and M.~Unser}, {\em Deep
  convolutional neural network for inverse problems in imaging}, IEEE
  Transactions on Image Processing, 26 (2017), pp.~4509--4522.

\bibitem{Kamilov2023plug}
{\sc U.~S. Kamilov, C.~A. Bouman, G.~T. Buzzard, and B.~Wohlberg}, {\em
  Plug-and-play methods for integrating physical and learned models in
  computational imaging: Theory, algorithms, and applications}, IEEE Signal
  Processing Magazine, 40 (2023), pp.~85--97.

\bibitem{Klee1960}
{\sc V.~Klee}, {\em Polyhedral sections of convex bodies}, Acta Mathematica,
  (1960), pp.~243--267.

\bibitem{Vedaldi2020}
{\sc B.~Lecouat, J.~Ponce, and J.~Mairal}, {\em Fully {Trainable} and
  {Interpretable} {Non}-local {Sparse} {Models} for {Image} {Restoration}}, in
  Proc. {ECCV} 2020, vol.~12367 of Lecture Notes in Computer Science, Springer
  International Publishing, 2020, pp.~238--254.

\bibitem{Liang2000principles}
{\sc Z.-P. Liang and P.~C. Lauterbur}, {\em Principles of Magnetic Resonance
  Imaging: A Signal Processing Perspective}, Wiley-IEEE Press, 2000.

\bibitem{Lin2021artificial}
{\sc D.~J. Lin, P.~M. Johnson, F.~Knoll, and Y.~W. Lui}, {\em Artificial
  intelligence for {MR} image reconstruction: an overview for clinicians},
  Journal of Magnetic Resonance Imaging, 53 (2021), pp.~1015--1028.

\bibitem{Lin2014}
{\sc J.~Lin and S.~Li}, {\em Sparse recovery with coherent tight frames via
  analysis {D}antzig selector and analysis {LASSO}}, Applied and Computational
  Harmonic Analysis, 37 (2014), pp.~126--139.

\bibitem{Lindenstrauss1966}
{\sc J.~Lindenstrauss}, {\em Notes on {K}lee's paper ``{P}olyhedral sections of
  convex bodies''}, Israel Journal of Mathematics, 4 (1966), pp.~235--242.

\bibitem{Mccann2017Convolutional}
{\sc M.~McCann, K.~Jin, and M.~Unser}, {\em Convolutional neural networks for
  inverse problems in imaging---{A} review}, {IEEE} Signal Processing Magazine,
  34 (2017), pp.~85--95.

\bibitem{McCann2019}
{\sc M.~McCann and M.~Unser}, {\em Biomedical image reconstruction: {F}rom the
  foundations to deep neural networks}, Foundations and Trends in Signal
  Processing, 13 (2019), pp.~280--359.

\bibitem{McMullen1970}
{\sc P.~McMullen}, {\em Polytopes with centrally symmetric faces}, Israel
  Journal of Mathematics, 8 (1970), pp.~194--196.

\bibitem{McMullen1971}
\leavevmode\vrule height 2pt depth -1.6pt width 23pt, {\em On zonotopes},
  Transactions of the American Mathematical Society, 159 (1971), pp.~91--109.

\bibitem{Megginson2012introduction}
{\sc R.~E. Megginson}, {\em An Introduction to {B}anach Space Theory},
  vol.~183, Springer Science \& Business Media, 2012.

\bibitem{Monga2021}
{\sc V.~Monga, Y.~Li, and Y.~C. Eldar}, {\em Algorithm unrolling:
  Interpretable, efficient deep learning for signal and image processing}, IEEE
  Signal Processing Magazine, 38 (2021), pp.~18--44.

\bibitem{Muckley2021MRIChallenge}
{\sc M.~J. Muckley, B.~Riemenschneider, A.~Radmanesh, S.~Kim, G.~Jeong, J.~Ko,
  Y.~Jun, H.~Shin, D.~Hwang, M.~Mostapha, S.~Arberet, D.~Nickel, Z.~Ramzi,
  P.~Ciuciu, J.-L. Starck, J.~Teuwen, D.~Karkalousos, C.~Zhang, A.~Sriram,
  Z.~Huang, N.~Yakubova, Y.~W. Lui, and F.~Knoll}, {\em Results of the 2020
  fast{MRI} challenge for machine learning {MR} image reconstruction}, IEEE
  Transactions on Medical Imaging, 40 (2021), pp.~2306--2317.

\bibitem{Mukherjee2020learned}
{\sc S.~Mukherjee, S.~Dittmer, Z.~Shumaylov, S.~Lunz, O.~{\"O}ktem, and C.-B.
  Sch{\"o}nlieb}, {\em Learned convex regularizers for inverse problems}, arXiv
  preprint arXiv:2008.02839,  (2020).

\bibitem{Mukherjee2023}
{\sc S.~Mukherjee, A.~Hauptmann, O.~{\"O}ktem, M.~Pereyra, and C.-B.
  Sch{\"o}nlieb}, {\em Learned reconstruction methods with convergence
  guarantees: A survey of concepts and applications}, IEEE Signal Processing
  Magazine, 40 (2023), pp.~164--182.

\bibitem{Narici2010}
{\sc L.~Narici and E.~Beckenstein}, {\em Topological Vector Spaces}, no.~296 in
  Pure and Applied Mathematics, CRC Press, Boca Raton, Fl, 2nd~ed., 2010.

\bibitem{Nataraj2020}
{\sc G.~Nataraj and R.~Otazo}, {\em Model-free deep {MRI} reconstruction: {A}
  robustness study}, in {ISMRM} {Workshop} on {Data} {Sampling} and {Image},
  2020.

\bibitem{Natterer1984}
{\sc F.~Natterer}, {\em The Mathematics of Computed Tomography}, John Willey \&
  Sons Ltd., 1984.

\bibitem{Nguyen2018}
{\sc H.~Nguyen, E.~Bostan, and M.~Unser}, {\em Learning convex regularizers for
  optimal {B}ayesian denoising}, {IEEE} Transactions on Signal Processing, 66
  (2018), pp.~1093--1105.

\bibitem{Papyan2017}
{\sc V.~Papyan, J.~Sulam, and M.~Elad}, {\em Working locally thinking globally:
  Theoretical guarantees for convolutional sparse coding}, IEEE Transactions on
  Signal Processing, 65 (2017), pp.~5687--5701.

\bibitem{Raguet2018}
{\sc H.~Raguet}, {\em A note on the forward-{D}ouglas/{R}achford splitting for
  monotone inclusion and convex optimization}, Optimization Letters, 13 (2018),
  pp.~717--740.

\bibitem{Ravishankar2019image}
{\sc S.~Ravishankar, J.~C. Ye, and J.~A. Fessler}, {\em Image reconstruction:
  From sparsity to data-adaptive methods and machine learning}, Proceedings of
  the IEEE, 108 (2019), pp.~86--109.

\bibitem{ReyOtero2020}
{\sc I.~Rey-Otero, J.~Sulam, and M.~Elad}, {\em Variations on the convolutional
  sparse coding model}, IEEE Transactions on Signal Processing, 68 (2020),
  pp.~519--528.

\bibitem{Roth2009}
{\sc S.~Roth and M.~J. Black}, {\em Fields of experts}, International Journal
  of Computer Vision, 82 (2009), pp.~205--229.

\bibitem{Rubinstein2010}
{\sc R.~Rubinstein, A.~M. Bruckstein, and M.~Elad}, {\em Dictionaries for
  sparse representation modeling}, Proceedings of the IEEE, 98 (2010),
  pp.~1045--1057.

\bibitem{Ryu2019plug}
{\sc E.~Ryu, J.~Liu, S.~Wang, X.~Chen, Z.~Wang, and W.~Yin}, {\em
  {Plug-and-play methods provably converge with properly trained denoisers}},
  in International Conference on Machine Learning, PMLR, 2019, pp.~5546--5557.

\bibitem{Schneider2013}
{\sc R.~Schneider}, {\em Convex Bodies: {T}he {B}runn-{M}inkowski Theory},
  Cambridge University Press, Oct. 2013.

\bibitem{Shephard1974}
{\sc G.~C. Shephard}, {\em Combinatorial properties of associated zonotopes},
  Canadian Journal of Mathematics, 26 (1974), pp.~302--321.

\bibitem{Shilov1977}
{\sc G.~E. Shilov}, {\em Linear Algebra}, Dover Publications, Inc., 1977.

\bibitem{Sun2021}
{\sc Y.~Sun, Z.~Wu, X.~Xu, B.~Wohlberg, and U.~S. Kamilov}, {\em Scalable
  plug-and-play {ADMM} with convergence guarantees}, IEEE Transactions on
  Computational Imaging, 7 (2021), pp.~849--863.

\bibitem{Unser2025a}
{\sc M.~Unser and S.~Ducotterd}, {\em Parseval convolution operators and neural
  networks}, arXiv preprint arXiv:2408.09981 [eess.SP],  (2024).

\bibitem{Unser2025ACHA}
{\sc M.~Unser, A.~Goujon, and S.~Ducotterd}, {\em Controlled learning of
  pointwise nonlinearities in neural-network-like architectures}, Applied and
  Computational Harmonic Analysis, 77 (2025), p.~101764.

\bibitem{Venkatakrishnan2013plug}
{\sc S.~V. Venkatakrishnan, C.~A. Bouman, and B.~Wohlberg}, {\em Plug-and-play
  priors for model based reconstruction}, in 2013 IEEE Global Conference on
  Signal and Information Processing, 2013, pp.~945--948.

\bibitem{Wang2020}
{\sc G.~Wang, J.~C. Ye, and B.~De~Man}, {\em Deep learning for tomographic
  image reconstruction}, Nature Machine Intelligence, 2 (2020), pp.~737--748.

\bibitem{Ye2023deep}
{\sc J.~C. Ye, Y.~C. Eldar, and M.~A. Unser}, {\em Deep Learning for Biomedical
  Image Reconstruction}, Cambridge University Press, 2023.

\bibitem{Ziegler2012Polytopes}
{\sc G.~M. Ziegler}, {\em Lectures on Polytopes}, vol.~152, Springer Science \&
  Business Media, 2012.

\end{thebibliography}


\end{document}